\documentclass[10pt,twocolumn,letterpaper]{article}

\usepackage[pagenumbers]{cvpr}              %

\makeatletter
\@namedef{ver@everyshi.sty}{}
\makeatother

\usepackage{comment}
\usepackage[vlined,ruled,linesnumbered]{algorithm2e}
\usepackage{graphics} %
\usepackage{rotating}
\usepackage{color}
\usepackage{enumerate}
\usepackage[T1]{fontenc}
\usepackage{psfrag}
\usepackage{epsfig} %
\usepackage{booktabs}
\usepackage{graphicx,url}
\usepackage{multirow}
\usepackage{array}
\usepackage{latexsym}
\usepackage{amsfonts}
\usepackage{amsmath}
\usepackage{amssymb}
\usepackage{amsthm}
\usepackage{xcolor}
\usepackage{prettyref}
\usepackage{flexisym}
\usepackage{bigdelim}
\usepackage{breqn} %
\usepackage{listings}
\usepackage{enumitem}
\usepackage{xspace}
\usepackage{bm}
\graphicspath{{./figures/}}
\usepackage{tikz}
\usetikzlibrary{matrix,calc}

\usepackage{amsthm,amsmath}

\newtheorem{theorem}{Theorem}

\newtheorem{lemma}[theorem]{Lemma}

\newtheorem{remark}[theorem]{Remark}

\newcommand{\bdmath}{\begin{dmath}}
\newcommand{\edmath}{\end{dmath}}
\newcommand{\beq}{\begin{equation}}
\newcommand{\eeq}{\end{equation}}
\newcommand{\bdm}{\begin{displaymath}}
\newcommand{\edm}{\end{displaymath}}
\newcommand{\bea}{\begin{eqnarray}}
\newcommand{\eea}{\end{eqnarray}}
\newcommand{\beal}{\beq \begin{array}{ll}}
\newcommand{\eeal}{\end{array} \eeq}
\newcommand{\beas}{\begin{eqnarray*}}
\newcommand{\eeas}{\end{eqnarray*}}
\newcommand{\ba}{\begin{array}}
\newcommand{\ea}{\end{array}}
\newcommand{\bit}{\begin{itemize}}
\newcommand{\eit}{\end{itemize}}
\newcommand{\ben}{\begin{enumerate}}
\newcommand{\een}{\end{enumerate}}

\makeatletter

\newcommand{\calF}{{\cal F}}

\newcommand{\calI}{{\cal I}}

\newcommand{\setS}{\textsf{S}}

\newcommand{\setal}{~\emph{et~al.}\xspace}

\newcommand{\M}[1]{{\bm #1}} %
\renewcommand{\boldsymbol}[1]{{\bm #1}}

\newcommand{\hide}[1]{}

\newcommand{\hiddenText}{{\color{gray} hidden text.}}
\newcommand{\hideWithText}[1]{\hiddenText}

\newcommand{\one}{ {\mathbf{1}} }

\newcommand{\norm}[1]{\left\| #1 \right\|}

\newcommand{\tran}{^{\mathsf{T}}}

\newcommand{\diag}[1]{\mathrm{diag}\left(#1\right)}

\newcommand{\ones}{{\mathbf 1}}

\newcommand{\Real}[1]{ { {\mathbb R}^{#1} } }

\newcommand{\SOthree}{\ensuremath{\mathrm{SO}(3)}\xspace}

\newcommand{\MD}{\M{D}}

\newcommand{\MR}{\M{R}}

\newcommand{\MF}{\M{F}}

\newcommand{\MX}{\M{X}}

\newcommand{\MZ}{\M{Z}}

\newcommand{\vh}{\boldsymbol{h}}

\newcommand{\vc}{\boldsymbol{c}}

\newcommand{\vt}{\boldsymbol{t}}
\newcommand{\vxx}{\boldsymbol{x}}

\newcommand{\vzz}{\boldsymbol{z}}

\newcommand{\blue}[1]{{\color{blue}#1}}

\newcommand{\linkToPdf}[1]{\href{#1}{\blue{(pdf)}}}
\newcommand{\linkToPpt}[1]{\href{#1}{\blue{(ppt)}}}
\newcommand{\linkToCode}[1]{\href{#1}{\blue{(code)}}}
\newcommand{\linkToWeb}[1]{\href{#1}{\blue{(web)}}}
\newcommand{\linkToVideo}[1]{\href{#1}{\blue{(video)}}}
\newcommand{\linkToMedia}[1]{\href{#1}{\blue{(media)}}}
\newcommand{\award}[1]{\xspace} %

\newcommand{\vz}{\boldsymbol{z}}

\newcommand{\bftab}{\fontseries{b}\selectfont}

\usepackage{wrapfig}
\usepackage[skip=4pt]{caption}
\usepackage{subcaption}
\usepackage{xr}
\usepackage{lipsum}

\definecolor{cvprblue}{rgb}{0.21,0.49,0.74}
\usepackage[pagebackref,breaklinks,colorlinks,allcolors=cvprblue]{hyperref}
\usepackage{colortbl}
\usepackage{makecell}
\usepackage{wrapfig}
\usepackage[]{microtype}

\newcommand{\name}{\texttt{CRISP}\xspace}
\newcommand{\pipelineNameST}{\pipelineName-\texttt{ST}\xspace}
\newcommand{\pipelineName}{\texttt{CRISP}\xspace}

\newcommand{\PNC}{PNC\xspace}
\newcommand{\PGD}{BCD\xspace}
\newcommand{\LSQ}{LSQ\xspace}

\newcommand{\RGBD}{RGB-D\xspace}
\newcommand{\DPT}{DPT\xspace}

\newcommand{\NOCS}{NOCS\xspace}
\newcommand{\NOCSREAL}{REAL275\xspace}
\newcommand{\NOCSREALTwoSevenFive}{REAL275\xspace}
\newcommand{\NOCSCAMERA}{CAMERA\xspace}

\newcommand{\pipelineNameReal}{\texttt{CRISP-Real}\xspace}
\newcommand{\pipelineNameSyn}{\texttt{CRISP-Syn}\xspace}
\newcommand{\pipelineNameSynLSQAdapt}{\texttt{CRISP-Syn-ST (LSQ)}\xspace}
\newcommand{\pipelineNameSynPGDAdapt}{\texttt{CRISP-Syn-ST (BCD)}\xspace}

\newcommand{\ocx}{\ensuremath{\texttt{oc}}\xspace}

\newcommand{\indicator}[1]{\ensuremath{\mathbb{I}\left\{ #1 \right\}}}

\newcommand{\omitt}[1]{ }

\definecolor{tfirst}{rgb}{0.40,0.76,0.64} 
\definecolor{tsecond}{rgb}{0.62,0.87,0.87}
\definecolor{tthird}{rgb}{0.88,0.97,0.98}
\newcommand{\tabfirst}{\cellcolor{tfirst}\bftab}  %
\newcommand{\tabsecond}{\cellcolor{tsecond}}  %
\newcommand{\tabthird}{\cellcolor{tthird}}  %
\newcommand{\tabfirstbox}{\colorbox{tfirst}{\rule{0pt}{2pt}\rule{2pt}{0pt}}}
\newcommand{\tabsecondbox}{\colorbox{tsecond}{\rule{0pt}{2pt}\rule{2pt}{0pt}}}
\newcommand{\tabthirdbox}{\colorbox{tthird}{\rule{0pt}{2pt}\rule{2pt}{0pt}}}

\def\paperID{17128} %
\def\confName{CVPR}
\def\confYear{2025}

\title{\pipelineName: Object Pose and Shape Estimation with Test-Time Adaptation}

\author{Jingnan Shi, Rajat Talak, Harry Zhang, David Jin, and Luca Carlone\\
Laboratory for Information \& Decision Systems (LIDS) \\
Massachusetts Institute of Technology \\
{\tt\small \{jnshi, talak, harryz, jindavid, lcarlone\}@mit.edu}
}

\begin{document}
\maketitle

\begin{abstract}
We consider the problem of estimating object pose and shape from an RGB-D image.
Our first contribution is to 
 introduce \pipelineName, a category-agnostic object pose and shape estimation pipeline. The pipeline implements an encoder-decoder model for shape estimation. It uses FiLM-conditioning for implicit shape reconstruction and a DPT-based network for estimating pose-normalized points for pose estimation. 
As a second contribution, 
we propose an optimization-based pose and shape corrector that can correct estimation errors caused by a domain gap.
Observing that the shape decoder is well behaved in the convex hull of known shapes, 
we approximate the shape decoder with an active shape model, and show that this reduces the shape correction problem to a constrained linear least squares problem, which can be solved efficiently by an interior point algorithm.   
Third, 
we introduce %
a self-training pipeline to perform self-supervised domain adaptation of \pipelineName. %
The self-training is based on a correct-and-certify approach, which leverages the corrector to generate pseudo-labels at test time, and uses them to self-train \pipelineName.
We demonstrate \name (and the self-training) on YCBV, SPE3R, and NOCS datasets. \name shows high performance on all the datasets. 
Moreover, our self-training is capable of bridging a large domain gap. Finally, \name also shows an ability to generalize to unseen objects.  
Code and pre-trained models will be available on the project webpage.\footnote{\url{https://web.mit.edu/sparklab/research/crisp_object_pose_shape/}}

\end{abstract}

\section{Introduction}
\label{sec:intro} 

\begin{figure}[t!]
    \centering
    \includegraphics[width=\linewidth]{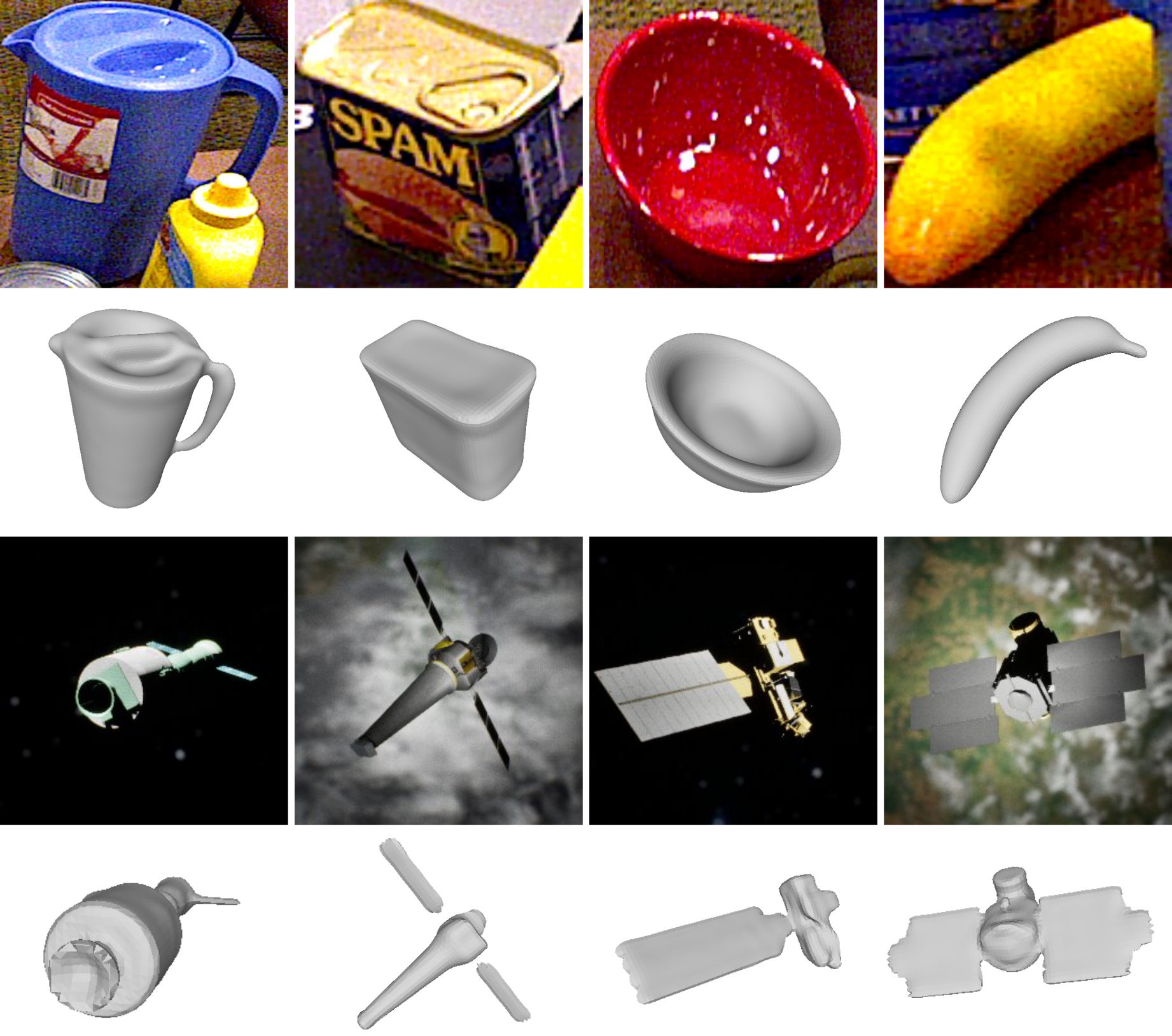}
    \caption{We introduce \pipelineName, a category-agnostic object pose and shape estimation pipeline, and a test-time adaptive self-training method \pipelineNameST to bridge domain gaps. 
    \emph{Top:} Qualitative examples of \pipelineName on the YCBV dataset~\cite{Xiang17rss-posecnn}.
    \emph{Bottom:} Qualitative examples of \pipelineName on the SPE3R dataset~\cite{Park24aiaa-spe3r}. }
    \label{fig:intro}
    \vspace{-5mm}
\end{figure}

Accurately estimating the geometry %
of objects is necessary in
many vision applications ranging from augmented reality~\cite{Zhao24arxiv-egopressure} to robotics~\cite{Hagele16springer-industrialRobotics},
and autonomous docking in space~\cite{Chen19ICCVW-satellitePoseEstimation}.
While significant progress has been made in solving object perception~\cite{Xiang17rss-posecnn,Wang19-normalizedCoordinate,Ze22neurips-wild6d},
real-world deployments remain challenging due to lack of generalizable models and in-domain training data~\cite{Sanneman21ftr-stateIndustrialRobot}.

Existing approaches to object perception tasks focus on instance-level 6D pose estimation~\cite{Labbe20eccv-CosyPose,Labbe22corl-megapose}, category-level object 6D pose and shape estimation~\cite{Tian20eccv-SPD, Ze22neurips-wild6d}, 3D bounding box estimation~\cite{Brazil23cvpr-omni3d,Mousavian17cvpr-3dBbox}, and more recently category-agnostic object pose and shape estimation~\cite{Lunayach24icra-FSD}.
Approaches for the first two tasks usually require object CAD model or category-level priors at inference time. 
On the other hand, 3D bounding box may be too coarse a representation to be useful in many applications (\eg autonomous space debris removal~\cite{Park24aiaa-spe3r}).  
Recent work on conditional diffusion models have shown promise in single-view 3D reconstruction~\cite{Jun23arxiv-shap-e, Liu24neuips-oneTwoThree}. However, their inference time remains prohibitively high for real-time operation~\cite{Liu24neuips-oneTwoThree}.

In addition, the domain gap issue, where the distribution of the test data differs from training, remains a major concern.
Developing an object perception model is not enough as a domain gap can render its estimates  completely useless, or worse still, hazardous in safety-critical applications~\cite{Sanneman21ftr-stateIndustrialRobot}. 
Therefore, in conjunction to the object perception model, it is necessary to develop methods that can bridge any large domain gap and work well at test-time.

Towards addressing these issues, this paper presents three main contributions:
\begin{itemize}
	\item We introduce \name, an object pose and shape estimation pipeline. \name combines a pre-trained vision transformer (ViT) backbone with a dense prediction transformer (DPT) and feature-wise linear modulation (FiLM) conditiong to estimate the 6D pose and shape of the 3D object from a single \RGBD image~\cite{Ranftl21iccv-DPT, Perez18aaai-FiLMVisual}. \name is category-agnostic 
	(\ie, it does not require knowledge of the object category at test time). 
	\item We propose an optimization-based pose and shape corrector that can correct estimation errors. The corrector is a bi-level optimization problem and we use block coordinate descent to solve it. 
		We approximate the shape decoder in \name by an active shape model, and show that (i) this is a reasonable approximation, and (ii) doing so turns it into a constrained linear least squares problem, which can be solved efficiently using interior-point methods 
		and yields just as good shapes as the trained decoder. 

	\item We adapt a \emph{correct-and-certify} approach to self-train \name and bridge any large domain gap. 
		During self-training,
		we use the corrector to correct for pose and shape estimation errors.
		Then, we assert the quality of the output of the corrector using an observable correctness certificate inspired by~\cite{Shi23rss-ensemble}, and create pseudo-labels using the estimates that pass the certificate check.
		Finally, we train the model on these pseudo-labels with standard stochastic gradient descent. Contrary to~\cite{Peng22aaai-selfSupPoseShape,Lunayach24icra-FSD}, we do not need access to synthetic data during self-training.

\end{itemize}

We demonstrate \name on the YCBV, SPE3R, and NOCS datasets. \name shows high performance on all the datasets.
Moreover, our self-training is able to bridge a large domain gap.%
Finally, we observe that \name is able to generalize and estimate pose and shape of objects unseen during training, and is fast for real-time inference.

\section{Related Work}
\label{sec:related_work}

\paragraph{Object Pose and Shape Estimation.}
The task of object pose and shape estimation is to recover the 3D poses and shape of an object under observation.
This is different from the task of instance-level pose estimation,
which usually assumes object CAD models are provided during test-time inference~\cite{Wang19nips-prnet, Wang21cvpr-GDRNetGeometryGuided, Xiang17rss-posecnn, Labbe20eccv-CosyPose, Labbe22corl-megapose, Li18eccv-DeepIMDeep, Deng24cvpr-unsupervised}.
Category-level approaches have been proposed, which aim to simultaneously estimate the shape and pose of the objects, despite large intra-class shape variations.
One common paradigm in category-level pose and shape estimation is to first regress some intermediate representations, and solve for poses and shapes.
Pavlakos\setal~\cite{Pavlakos17icra-semanticKeypoints} use a stacked hourglass neural network~\cite{Newell16-stackedHourglass} for 2D semantic keypoint detection.
Wang\setal propose Normalized Object Coordinates (NOCS) to represent the objects as well as use them to recover pose and shape~\cite{Wang19-normalizedCoordinate}. 
Many recent works adopt category-specific shape priors, which are deformed to match the observation during inference.
Tian\setal~\cite{Tian20eccv-SPD} propose to learn the category mean shapes from training data, and then use them for pose and shape estimation.
Ze and Wang~\cite{Ze22neurips-wild6d} propose RePoNet, which trains the pose and shape branches using differentiable rendering.
However, the use of category-level priors limit the scalability of these approaches.

Recently, there have been a number of works using neural implicit fields for object pose and shape estimation~\cite{Lunayach24icra-FSD, Peng22aaai-selfSupPoseShape, Irshad22eccv-shapo}.
Peng\setal~\cite{Peng22aaai-selfSupPoseShape} combine PointNet~\cite{Qi17cvpr-pointnet} with a DeepSDF decoder~\cite{Park19cvpr-deepSDF} for representing shapes.
Irshad\setal~\cite{Irshad22eccv-shapo} learn a disentangled implicit shape and apperance model for joint appearance, shape and pose optimization.
Comparing with conditional diffuion models that are prevalent in the shape reconstruction literature~\cite{Jun23arxiv-shape, Liu24neuips-oneTwoThree},
such regression-based methods are significantly faster for inference, making 
them suitable for real-time applications.

\paragraph{Test-Time Adaptation.} Test-time adaptation allows methods to handle distribution shifts between training and testing, 
leading to better performance on unseen domains~\cite{Xiao24arxiv-beyondAdaptation}.
Wang\setal~\cite{Wang20eccv-Self6DSelfSupervised} train a pose estimation model on synthetic RGB-D data, and then refine it further with differentiable rendering on real, unannotated data.
However, the approach is limited to objects with known CAD models.
Lunayach\setal~\cite{Lunayach24icra-FSD} use a self-supervised chamfer loss to provide training signals on unannotated data. 
However, it requires access to a set of synthetic labeled data during self-supervision to stablize training, making it less suitable for test-time adaption in real-world robotics deployments.
Talak\setal~\cite{Talak23tro-c3po} propose a correct-and-certify approach to self-train pose estimation models based on a corrector and binary certificates, which is further extended in~\cite{Shi23rss-ensemble} to handle outliers. 
However, the approach assumes known CAD models at test-time.

\section{Object Pose and Shape Estimation}
\label{sec:arch}
\begin{figure*}[t!]
\centering
\includegraphics[width=\linewidth]{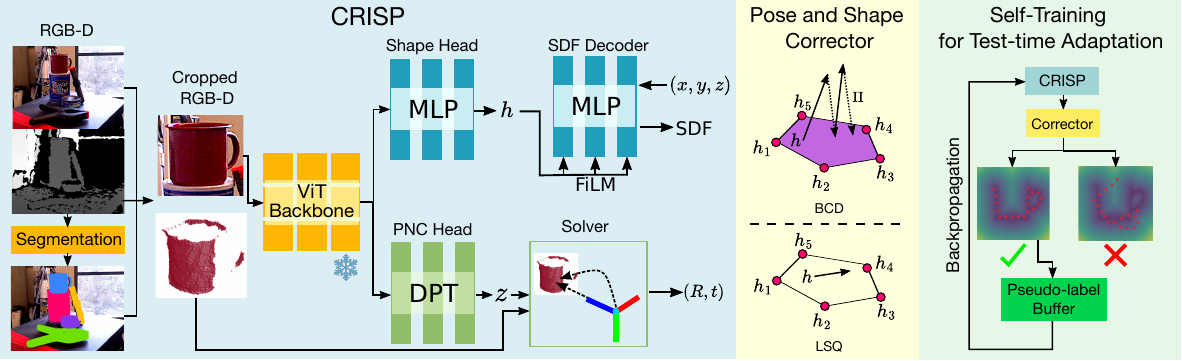}
\caption{
Overview of our contributions. 
Given an segmented RGB image $\calI$ and depth points $\MX$ of the object, \pipelineName %
extracts features from the cropped image. It estimates the object pose, by estimating pose-normalized coordinates (\PNC) $\MZ$, and shape, by reconstructing the signed distance field (SDF) of the object. The pose and shape estimates are corrected by the corrector which solves a bi-level optimization problem using two solvers: \PGD (Alg.~\ref{alg:corrector}) and \LSQ (Alg.~\ref{alg:corrector-lsq}). The self-training uses corrected estimates that pass the observable certification check~\eqref{eq:oc-cert} as pseudo-labels. The SDF decoder is fixed during self-training.}
\label{fig:pipeline}
\vspace{-5mm}
\end{figure*}

We consider the problem of object pose and shape estimation from an RGB-D image $(\calI, \MX)$; where $\calI$ is the RGB image of the detected object and $\MX = [\vxx_1, \vxx_2, \ldots \vxx_n] \in \Real{3\times n}$ is the depth point cloud of the segmented object. 
Each point $\vxx_i$ corresponds to a pixel in the image $\calI$.
Given $(\calI, \MX)$, our goal is to estimate the shape $f(\cdot~|~\calI)$ and pose $(\MR, \vt) \in \SOthree\times \Real{3}$ of the object, where $f(\cdot~|~\calI)$ is the signed distance field (SDF) of the object in a pose-normalized frame. We can write this as an optimization problem:
\begin{equation}
\begin{aligned}
\label{eq:original}
& \underset{f, \MR, \vt}{\text{Minimize}}
& & \sum_{i=1}^{n} |f(\MR \vxx_i + \vt~|~\calI)|^2\\
& \text{subject to}
& & (\MR, \vt) \in \SOthree\times \Real{3} \\
&&& f(\cdot~|~\calI) \in \calF.
\end{aligned}
\end{equation}
where the objective forces the depth points measured on the object surface to match the zero-crossing of the SDF.
This problem is hard because the space of shapes $\calF$ and the input $\calI$ are not easy to characterize. Secondly, the problem is not well specified, if $\calF$ is not sufficiently constrained by a  set of priors. 
Thirdly, estimating $(\MR, \vt)$, by solving~\eqref{eq:original} is difficult as the objective is non-linear and non-convex, and the set of object poses, \ie, $\SOthree \times \Real{3}$, is a manifold.  

We now present how learning-based components can be used to obviate these issues and approximate the optimization problem~\eqref{eq:original} by performing forward passes on the learning-based components. %
This forms
our category-agnostic pose and shape estimation pipeline \pipelineName (Fig.~\ref{fig:pipeline}).

\subsection{Shape Estimation}
\label{sec:arch-shape}
We use a deep encoder-decoder architecture to estimate the object shape from an image (Fig.~\ref{fig:pipeline}). An encoder network estimates a latent shape code $\vh = f_e(\calI) \in \Real{d}$, from the image $\calI$. A shape decoder produces an SDF $f_d(\cdot~|~\vh)$ of the object, given the latent shape code $\vh$. The shape $f_d(\cdot~|~\vh)$ is in a pose-normalized frame, \ie, centered at origin and is invariant to the pose of the object in the image.

\emph{Architecture Details.} Given an \RGBD image, 
we first extract features tokens with a pre-trained ViT backbone (\eg, DINOv2~\cite{Oquab23arxiv-dinov2}).
These are concatenated and passed through a multi-layer perceptron (MLP) to regress the latent shape code $\vh$. This forms our encoder $\vh = f_e(\calI)$. 
The shape decoder is a MLP with sinusoidal activations~\cite{Sitzmann20neurips-siren}. We use FiLM~\cite{Perez18aaai-FiLMVisual} to condition the decoder.
FiLM conditioning allows us to condition the SDF without any explicit category labels. This has been shown to produce better reconstruction results for neural implicit fields~\cite{Chan21cvpr-PiGAN}, in comparison to other popular methods like conditioning-by-concatenation.

If the encoder estimates the correct shape, 
the pose and shape estimation problem~\eqref{eq:original} reduces to
\begin{equation}
\begin{aligned}
\label{eq:original-pose}
& \underset{\MR, \vt}{\text{Minimize}}
& & \sum_{i=1}^{n} |f_d(\MR \vxx_i + \vt~|~\vh)|^2\\
& \text{subject to}
& & (\MR, \vt) \in \SOthree\times \Real{3}, \vh = f_e(\calI),
\end{aligned}
\end{equation}
where the pose estimation only depends on the decoder network $f_d$.
This is still a hard problem as it requires one to optimize a non-convex function over a manifold.

\subsection{Pose Estimation}
\label{sec:arch-pose}
We now use another network to simplify the pose estimation problem~\eqref{eq:original-pose}.
Inspired by~\cite{Wang19-normalizedCoordinate}, we use a deep neural network to estimate, for each pixel and depth point $\vxx_i$, a corresponding point $\vzz_i$ in the pose-normalized frame. We estimate these pose normalized coordinates (\PNC) $\MZ = [\vz_1, \vz_2, \ldots \vz_n]$ directly from the image, \ie
\begin{equation}
	\MZ = [\phi_1(\calI), \phi_2,(\calI), \ldots \phi_n(\calI)] = \Phi(\calI).
\end{equation}
 
\emph{Architecture Details.} We use a Dense Prediction Transformer (\DPT) to estimate the \PNC~\cite{Ranftl21iccv-DPT}.  
The \DPT uses feature tokens extracted from the image $\calI$ and passes them through 
trainable reassemble blocks. These are then progressively added and upsampled together to produce a fine-grained prediction.
The final output head is a CNN that directly regresses the \PNC $\MZ$. See Fig.~\ref{fig:pipeline}. 

If the estimated \PNC $\MZ$ are correct, then the object pose $(\MR, \vt)$, in~\eqref{eq:original-pose}, can be obtained by directly solving 
\begin{equation}
	\label{eq:pose}
	\vzz_i = \MR \vxx_i + \vt~~\forall~i \in [n],
\end{equation}
for $\MR \in \SOthree$ and $\vt \in \Real{3}$. We solve~\eqref{eq:pose} using Arun's method~\cite{Arun87pami}, which obtains $(\MR, \vt)$ by minimizing the least squares error in satisfying~\eqref{eq:pose}. %

\begin{remark}[Difference from NOCS~\cite{Wang19-normalizedCoordinate}]
Contrary to NOCS~\cite{Wang19-normalizedCoordinate}, we do not normalize $\MZ$. This is crucial for self-training: involving scale estimation leads to degeneracies of PNC, causing $\MZ$ to drift away from the ground truth (see Appendix~\ref{sec:app:add-exp}).
\end{remark}

\subsection{Supervised Training}
\label{sec:arch-supervised-training}
\pipelineName can be trained on both synthetic and real-world annotated data.
We require no category labels ---
any collection of CAD models can be used and they may belong to different categories. 
 We use a soft $L_1$ loss on \PNC $\MZ$ and the SDF loss specified in~\cite{Sitzmann20neurips-siren} (Appendix~\ref{sec:app:training-details}).
In Section~\ref{sec:st}, we will see how to use the corrector to self-train \pipelineName, when the trained model suffers a domain gap at test-time.

\section{Pose and Shape Correction}
\label{sec:llsq}
In Section~\ref{sec:arch}, we saw that the pose and shape estimation problem~\eqref{eq:original} is obviated if \pipelineName estimates the correct latent shape code and \PNC. 
In this section, we re-instate the optimization problem and tackle the case when the estimates from the learning-based components deviate from the ground-truth. 
We 
derive an optimization-based corrector that helps correct the estimated object pose and shape (Section~\ref{sec:arch-corrector}).
We then 
propose an active shape decoder, which simplifies the shape correction problem to a constrained linear least squares problem (Section~\ref{sec:llsq-active-shape-decoder}). This active shape decoder enables faster refinement (Section~\ref{sec:expt}). %

\subsection{Pose and Shape Corrector}
\label{sec:arch-corrector}
We now re-state the pose and shape optimization problem~\eqref{eq:original-pose}, and treat the latent shape code $\vh$ and the \PNC $\MZ$ (as in~\eqref{eq:pose}) as variables.
We obtain:
\begin{equation}
\begin{aligned}
\label{eq:corrector-01}
& \underset{\MZ, \vh, \MR, \vt}{\text{Minimize}}
& & \sum_{i=1}^{n} |f_d(~\vzz_i |~\vh)|^2\\
& \text{subject to} 
&& \vzz_i = \MR \vxx_i + \vt~\forall~i \in [n].
\end{aligned}
\end{equation}
Let $\vh^\ast$ and $\MZ^{\ast}$ denote the global optima of problem~\eqref{eq:corrector-01}. 
The goal of the corrector is to ensure that the network estimates, namely, 
\begin{equation}
\label{eq:model-estimates}
	\MZ = \Phi(\calI)~~~\text{and}~~~\vh = f_e(\calI),
\end{equation}
get corrected to $\MZ^\ast$ and $\vh^\ast$. 
\begin{remark}[Correcting to Ground Truth]
Assume that the shape decoder is well trained, \ie, for any object observed at test-time there is a latent shape code $\vh$ that generates the signed distance field of the object via $f_d(\cdot~|~\vh)$. Let $\vh^\ast$ and $\MZ^{\ast}$ denote the ground-truth latent shape code and \PNC, respectively. 
Then, it is trivial to see that the corrector objective will be zero, as $f_d(\vzz^\ast_i~|~\vh^\ast) = 0$ for all $i$.
\end{remark}
Ideally, one could solve~\eqref{eq:corrector-01} to global optimality without any network estimates. 
However, the problem~\eqref{eq:corrector-01} does not have an easy global solver. 
Therefore, we use the network estimates~\eqref{eq:model-estimates} as an initialization, followed by gradient descent (GD) to obtain a solution to~\eqref{eq:corrector-01}. If the model estimates provide a good initialization, the corrector will obtain a solution that is close to the global optima.

\emph{Corrector Implementation}. First, we simplify the corrector problem~\eqref{eq:corrector-01} by relaxing the hard constraints (\ie  $\vzz_i = \MR \vxx_i + \vt$). 
In Appendix~\ref{sec:app:opt}, we show that~\eqref{eq:corrector-01} is equivalent to the bi-level optimization problem:
\begin{equation}
\begin{aligned}
\label{eq:corrector-02}
& \underset{\MZ \in \Real{3\times n}, \vh \in \Real{d}}{\text{Minimize}}
& & \sum_{i=1}^{n} |f_d(\hat{\MR} \vxx_i + \hat{\vt}|~\vh)|^2\\
& \text{subject to} 
&& (\hat{\MR}, \hat{\vt}) = \underset{(\MR, \vt)}{\text{argmin}} \sum_{i=1}^{n}\norm{\MR\vxx_i + \vt - \vz_i}^2,
\end{aligned}
\end{equation}
such that the global optima of~\eqref{eq:corrector-01} can be obtained from the global optima of~\eqref{eq:corrector-02}. 
Note that $(\hat{\MR}, \hat{\vt})$ depend on $\MZ$. Denote $F(\MZ~|~\vh) = \sum_{i=1}^{n} |f_d(\hat{\MR} \vxx_i + \hat{\vt}|~\vh)|^2$ where  $(\hat{\MR}, \hat{\vt})$ are given by the constraint in~\eqref{eq:corrector-02}.
The bi-level optimization problem can be written as:
\begin{equation}
\begin{aligned}
\label{eq:corrector-03}
& \underset{\MZ \in \Real{3\times n}, \vh \in \Real{d}}{\text{Minimize}}
& & F(\MZ~|~\vh).
\end{aligned}
\end{equation}

The corrector solves this using model estimates as initialization~\eqref{eq:model-estimates} and block coordinate descent (\PGD); see Alg.~\ref{alg:corrector}. 
We first solve~\eqref{eq:corrector-03} using gradient descent, with $\MZ$ as variable and use network estimate as initialization (Alg.~\ref{alg:corrector}, Line~\ref{algo:nocs-update}).
We then solve~\eqref{eq:corrector-03} with $\vh$ as variable (initialized with network estimate) using projected gradient descent, initialize $\MZ$ with $\hat{\MZ}$ (Alg.~\ref{alg:corrector}, Line~\ref{algo:shape-update}).
The rationale behind using projection is due to our observation that the shape decoder is well behaved within the simplex, and not outside it.

\begin{figure*}[t!]
\centering
\includegraphics[width=\linewidth]{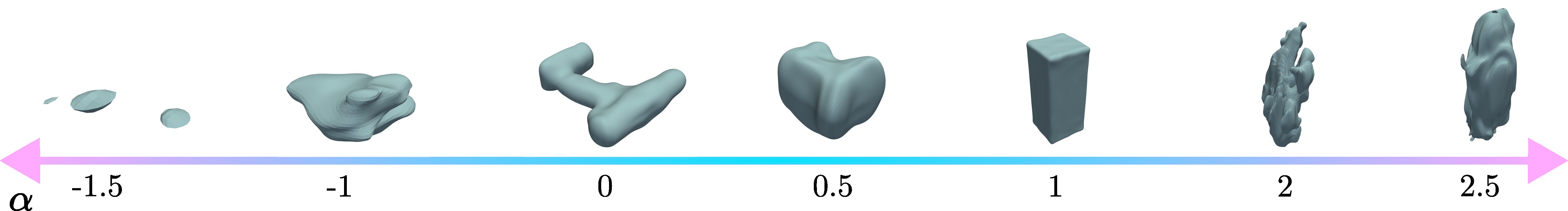}
\caption{Visualization of the mesh extracted from SDF produced by the shape decoder $f_d(\cdot~|~\vh)$ as the latent shape code takes values $\vh = \alpha \vh_1 + (1 - \alpha) \vh_2$ given two shapes codes $\vh_1$ and $\vh_2$. The trained decoder does not produce plausible shapes at extrapolation.}
\label{fig:shape-interpolation}
\vspace{-4mm}
\end{figure*}

To be more precise, let $\vh_1, \vh_2, \ldots \vh_K$, be the latent shape codes for the $K$ objects in the train set. 
Let $\vh$ take the values in the affine space
$
	\vh = \alpha_1 \vh_1 + \alpha_2 \vh_2 + \cdots \alpha_K \vh_K,	
$
with $\sum_{k=1}^{K} \alpha_k = 1$. When $\vh$ is interpolated (\ie, $1 \geq \alpha_k \geq 0$) the decoder $f_d(\cdot~|~\vh)$ is well behaved. 
When it is extrapolated, %
the decoder produces implausible object shapes (see Fig.~\ref{fig:shape-interpolation}).
Therefore, we use projected gradient descent (Line~\ref{algo:shape-update}) in Alg.~\ref{alg:corrector}. 
We project the iterates onto the simplex $\setS_K$ of all latent shape codes learnt in training~\cite{Shalev06jmlr-simplexProjection}:
\begin{equation}
\label{eq:shape-simplex}
	\setS_K = \left\{ \vh = \sum_{i=1}^{K} \alpha_k \vh_k~\Big|~\sum_{k=1}^{K} \alpha_k = 1~~\text{and}~~\alpha_k \geq 0~\right\}. \nonumber 
\end{equation}  
The \PGD algorithm is easily extended for multi-view pose and shape correction (Remark~\ref{rem:shape-multi-view}), which we implement in our experiments. 
\newlength{\textfloatsepsave} 
\setlength{\textfloatsepsave}{\textfloatsep} 
\setlength{\textfloatsep}{0pt}
\begin{algorithm}[t!]
\caption{\PGD: Block Coordinate Descent Solver for the Pose and Shape Correction~\eqref{eq:corrector-03}}
\label{alg:corrector}

$\MZ= \Phi(\calI), \vh = f_e(\calI)$\;

$\hat{\MZ} \leftarrow \underset{\MZ}{\text{argmin}}~F(\MZ~|~\vh)$ using Grad. Descent\; \label{algo:nocs-update}

$\hat{\vh} \leftarrow \underset{\vh}{\text{argmin}}~F(\hat{\MZ}~|~\vh)$; $\vh \in \setS_K = \left\{ \vh = \sum_{i=1}^{K} \alpha_k \vh_k~\Big|~\sum_{k=1}^{K} \alpha_k = 1~~\text{and}~~\alpha_k \geq 0~\right\}$ using Proj. Grad. Descent\; \label{algo:shape-update}

Return: $\hat{\vh}, \hat{\MZ}$\;

\end{algorithm}
\setlength{\textfloatsep}{\textfloatsepsave}

\begin{remark}[Multi-View Pose and Shape Corrector]
\label{rem:shape-multi-view}
	The corrector problem~\eqref{eq:corrector-01} and the \PGD algorithm is specified for a single input image, \ie, it corrects the pose and shape of an object, estimated from a single-view. However, the corrector problem and algorithm is easily modified to correct shape using estimates from multiple views. This is done by aggregating the estimated \PNC $\MZ$ across views, and using the $F(\MZ~|~\vh)$, with this aggregated $\MZ$, when updating the latent shape code $\vh$ in Line~\ref{algo:shape-update}, Algorithm~\ref{alg:corrector}. We implement both single-view and multi-view shape corrector. In the multi-view corrector, the \PNC get corrected for each view (\ie Line~\ref{algo:nocs-update}, Algorithm~\ref{alg:corrector}, remains constrained to single-view); however, using corrected shape code informed by multi-view data
	shows better results (Section~\ref{sec:expt}). 
\end{remark}

\subsection{Active Shape Decoder}
\label{sec:llsq-active-shape-decoder}
Inspired by our observation that the trained shape decoder $f_d(\cdot~|~\vh)$ is reliable only when the latent shape vector $\vh$ is in the simplex $\setS_K$, we now construct a new active shape decoder $f_a$, given $f_d$, that is reliable enough over the simplex $\setS_K$. 
This
reduces the corrector problem to a constrained linear least squares problem.

Let $\vh = f_e(\calI)$ be the latent shape code estimated by the model, and $\vh_1, \vh_2, \ldots \vh_K$ be the latent shape codes of all the CAD models used in the supervised training (Section~\ref{sec:arch-supervised-training}). We define an active shape decoder $f_a$ as:
\begin{equation}
	\label{eq:active-shape}
	f_a(\vz~|~\vc) = c_0 d_0 f_d(\vz~|~\vh) + \sum_{k=1}^{K} c_k d_k f_d(\vz~|~\vh_k),
\end{equation}
where $\vc = [c_0, c_1, \ldots c_K]\tran$ is a $(K+1)$-dimensional vector and $d_k$ are some positive constants. As in~\eqref{eq:shape-simplex}, we limit $\vc$ to be in the simplex, \ie, $\sum_{k=0}^{K} c_k = 1$ and $c_k \geq 0$.
The positive constants $d_k$ prove useful in normalizing the signed distance field
(see Appendix~\ref{sec:app:add-exp}).
Note that using latent shape code $\vh = f_e(\calI)$ as a ``basis'' in~\eqref{eq:active-shape} is essential for the approximation to work well. We 
empirically observe that using the active shape decoder without the $\vh$ leads to worse shape reconstruction performance (see Appendix~\ref{sec:app:add-exp}).

For the active shape model $f_a$, we now show that the shape correction~\eqref{eq:corrector-02}, given \PNC $\MZ$, reduces to a constrained linear least squares problem. 
This is useful as we can directly use a constrained linear least squares solver to update the shape code in \PGD (Line~\ref{algo:shape-update}, Algorithm~\ref{alg:corrector}).
First, note that, given $\MZ$, the shape estimation problem~\eqref{eq:corrector-02} for the active shape decoder $f_a$ can be written as  
\begin{equation}
\begin{aligned}
\label{eq:shape-corrector-01}
& \underset{\vc \in \Real{K+1}}{\text{Minimize}}
& & \sum_{i=1}^{n} |f_a(\vz_i~|~\vc)|^2\\
& \text{subject to}
& & \sum_{k=1}^{K} c_k = 1~~\text{and}~~c_k \geq 0.
\end{aligned}
\end{equation}
We obtain this by simply replacing the trained decoder $f_d(\cdot~|~\vh)$ with the active shape decoder $f_a$. 
Using~\eqref{eq:active-shape}, it is easy to show that the objective function in~\eqref{eq:shape-corrector-01} can be simplified as 
$
	\sum_{i=1}^{n} |f_a(\vz_i~|~\vc)|^2 = \norm{\MF(\MZ) \MD \vc}^2,
$
where $\MD = \diag{[d_0, d_1, \ldots d_K]}$ and 
\begin{equation}
	\MF(\MZ) = \left[ \begin{matrix}
		f_d(\vz_1|~\vh) & f_d(\vz_1|~\vh_1) & \cdots & f_d(\vz_1|~\vh_K) \\
		f_d(\vz_2|~\vh) & f_d(\vz_2|~\vh_1) & \cdots & f_d(\vz_2|~\vh_K) \\
		\vdots	& \vdots & & \vdots \\
		f_d(\vz_n|~\vh) & f_d(\vz_n|~\vh_1) & \cdots & f_d(\vz_n|~\vh_K) 
	\end{matrix} \right].
\end{equation}
This results in a new solver for the corrector problem~\eqref{eq:corrector-01}. This is described in Algorithm~\ref{alg:corrector-lsq} (\LSQ). We use interior point method to solve the constrained linear least squares problem in Line~\ref{algo:lsq-shape-update} of \LSQ~\cite{Diamond16cvxpy}.
\setlength{\textfloatsep}{0pt}
\begin{algorithm}[t]
\caption{\LSQ: Solver for the Pose and Shape Correction~\eqref{eq:corrector-03}.}
\label{alg:corrector-lsq}
$\MZ = \Phi(\calI), \vh = f_e(\calI)$\;
$\hat{\MZ} \leftarrow \underset{\MZ}{\text{argmin}}~F(\MZ~|~\vh)$ using Grad. Descent\; \label{algo:lsq-nocs-update}
$\hat{\vc} \leftarrow \underset{\vc \geq 0, \one\tran\vc = 1}{\text{argmin}} \norm{\MF(\hat{\MZ}) \MD \vc}^2$ using Interior Point\; \label{algo:lsq-shape-update}
$\hat{\vh} \leftarrow \hat{c}_0 d_0 \vh + \sum_{k=1}^{K} \hat{c}_k d_k \vh_k$\; 
Return: $\hat{\vh}$, $\hat{\MZ}$\;
\end{algorithm}
\setlength{\textfloatsep}{\textfloatsepsave}

\subsection{Shape Degeneracy Check}
\label{sec:llsq-degeneracy}
The linear least square shape corrector~\eqref{eq:shape-corrector-01} informs a degeneracy check on the shape estimator, i.e., a check if the shape estimation is unique or not.
From the shape corrector~\eqref{eq:shape-corrector-01} objective, we have
\begin{equation}
	\norm{\MF(\MZ) \MD \vc}^2 = \vc\tran\MD \MF(\MZ)\tran \MF(\MZ) \MD \vc. 
\end{equation}
It follows that if the matrix $\MF(\MZ)\tran \MF(\MZ)$ is singular then the shape corrector problem~\eqref{eq:shape-corrector-01} will have multiplicity of solutions. Thus for a given input, we can check the singularity of the matrix $\MF(\MZ)\tran \MF(\MZ)$ to determine shape estimation uniqueness.
Fig.~\ref{fig:degen} plots the minimum eigenvalue $\lambda_\min$ of the matrix $\MF(\MZ)\tran \MF(\MZ)$ as a function of keyframes $N$. 
In the example, as the handle of the mud begins to show up,
the minimum eigenvalue $\lambda_\min$ increases drastically. 
\begin{figure}[t!]
	\centering
	\includegraphics[width=0.9\linewidth]{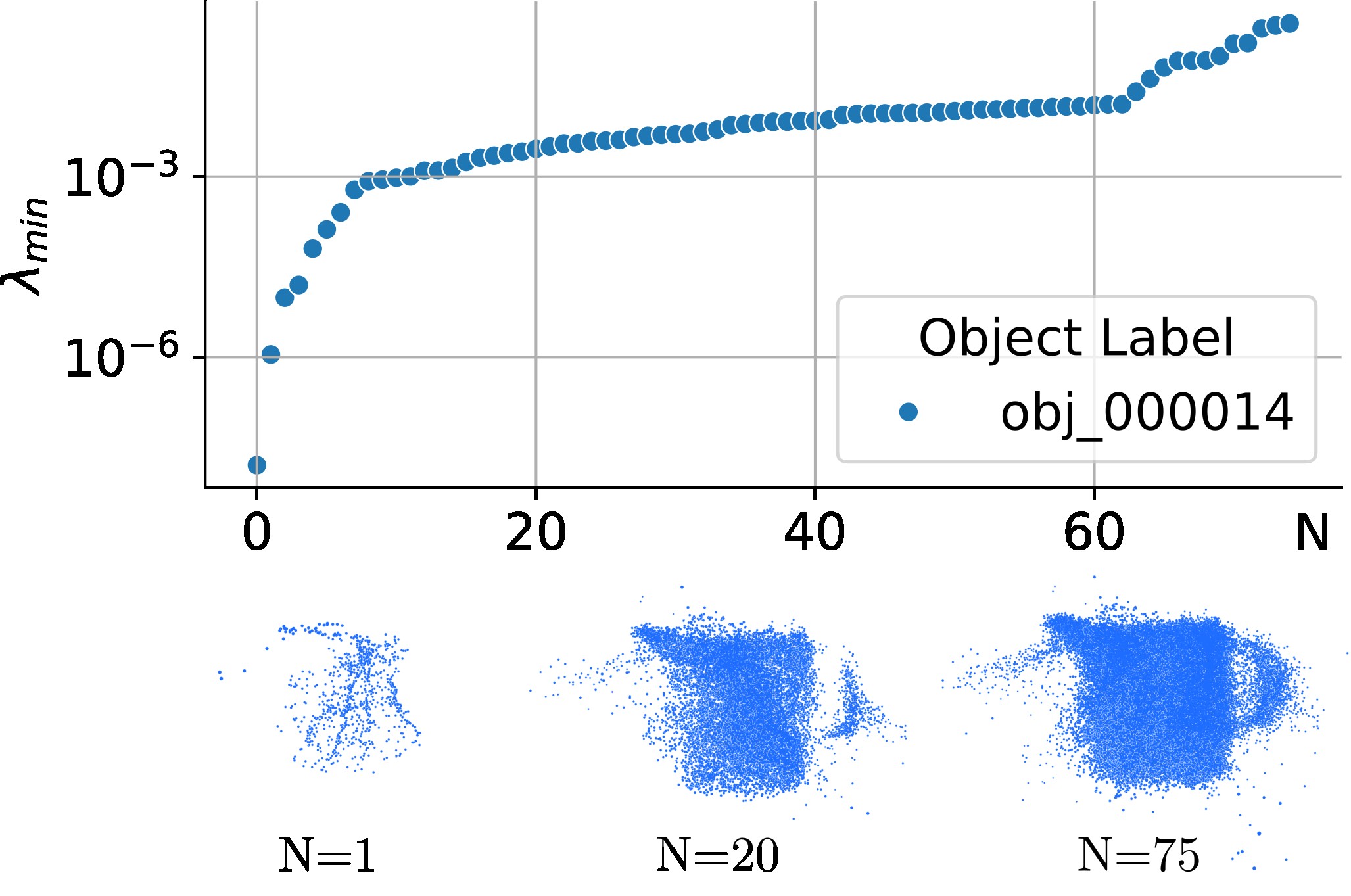}
	\caption{The minimum eigenvalue of the matrix $\MF(\MZ)\tran \MF(\MZ)$ as a function of keyframes $N$. Each keyframe captures the mug from a different viewing angle. $\MF(\MZ)$ is computed using the estimated \PNC $\MZ$, aggregated over all keyframes till $N$. }
	\label{fig:degen}
	\vspace{-5mm}
\end{figure}

\section{Self-Training for Test-Time Adaptation}
\label{sec:st}
We now describe \pipelineNameST, the self-training algorithm 
that self-trains \pipelineName at test-time with pesudo-labels. 
While self-training with pseudo-labels is not new~\cite{Talak23tro-c3po, Shi23rss-ensemble},
to the best of our knowledge, this is the first time that it has been applied to simultaneous object pose and shape estimation.
We now describe the three steps in \pipelineNameST: \emph{correction}, \emph{certification}, and \emph{self-training}.

\emph{Correction.} We correct the outputs~\eqref{eq:model-estimates} using the corrector (Section~\ref{sec:arch-corrector}) to produce refined estimates $\hat{\MZ}$ and $\hat{\vh}$.

\emph{Certification.} We implement observable correctness certificate checks to assert the quality of the estimates from the corrector. We do so by testing geometric consistency of the corrected estimates with the observed depth. 
 If $(\hat{\MR}, \hat{\vt})$ denote the corrected pose (see~\eqref{eq:corrector-02}), then the value $|f_d(\hat{\MR}\vxx_i + \hat{\vt}~|~\vh)|$
 must be low, barring outliers in the point cloud $\MX$. Intuitively, this ensures that the pose-inverted depth is consistent with the implicit shape representation $f_d(\cdot \vert \vh)$.Therefore, we define the observably correctness certificate (\ocx) as:
\begin{equation}
\label{eq:oc-cert}
    \texttt{oc}(\hat{\MZ}, \hat{\vh}) = \indicator{ \left[\lvert f_d(\hat{\vz}_i \vert \hat{\vh})) \rvert \right]_p < \epsilon}
\end{equation}
where $[\cdot ]_{p}$ is the $p$-th quantile function \wrt $i$. If the corrected estimates $(\hat{\MZ}, \hat{\vh})$ pass the \ocx check, then we self-annotate the image $\calI$ with the corrected estimates. 

\emph{Self-training.} We train the model using the self-annotated images with the MSE training loss on both the latent shape code and the \PNC:  
\begin{align}
    \label{eq:self-training-loss}
   L_h = \lVert \hat{\vh} - \vh \rVert^2 \qquad & L_z = \sum_{i=1}^{n} \lVert \hat{\vz_i} - \vz_i \rVert^2.
\end{align}
As the self-training progresses, more and more corrected estimates pass the \ocx check, resulting in more pseudo-labels. Note that the shape decoder is frozen during self-training.

\emph{Implementation Details}. The self-training pipeline operates over a buffer of images, including different views for a scene. \pipelineName operates on every object detection on each image. We implement a multi-view corrector that inputs data associations of objects in different images. We also report the results of using a single-view corrector in the ablation study (Sec.~\ref{sec:expt-ablations}).

\section{Experiments}
\label{sec:expt} 
We evaluate our approach with three datasets ranging from household objects to satellites:
(1) YCBV~\cite{Hodan20eccvw-BOPChallenge}, 
where the test and train objects are the same; 
(2) SPE3R~\cite{Park24aiaa-spe3r}, where the test objects are unknown;
(3) NOCS, 
a large-scale dataset with unknown test objects~\cite{Wang19-normalizedCoordinate}. 
We implement our models in PyTorch~\cite{Paszke19neurips-pytorch}.
We use DINOv2~\cite{Oquab23arxiv-dinov2} as the ViT backbone and keep it frozen during training.
All supervised models are trained with 8 Nvidia V100 GPUs,
while self-training experiments are conducted on one Nvidia A6000 GPU.
For tables, we use bold fonts for the best results and color the top three results with \tabfirstbox, \tabsecondbox~and \tabthirdbox, respectively. 
Additional details are provided in Appendix~\ref{sec:app:training-details}.

\subsection{YCBV Dataset}
\label{sec:expt-ycbv}

\emph{Setup.} %
We evaluate \pipelineName and our self-training approach against baselines on the YCBV dataset~\cite{Hodan20eccvw-BOPChallenge,Xiang17rss-posecnn}.
We train two models: \pipelineNameReal where we use the real-world images in the train set provided by the dataset, and \pipelineNameSyn where we train \pipelineName using 4200 (200 images per object) images rendered using BlenderProc~\cite{Denninger20rss-blenderproc} (see Appendix~\ref{sec:app:add-exp} for sample images).
The reasoning behind using synthetic images is to evaluate test-time adaptation under a large sim-to-real gap.
We train \pipelineNameReal with an Adam optimizer with a learning rate of $3 \times 10^{-4}$ for 50 epochs,
and \pipelineNameSyn with the same optimizer and learning rate for 500 epochs.
In addition, we self-train \pipelineNameSyn for 5 epochs,
using both \PGD Alg.~\ref{alg:corrector} (\pipelineNameSynPGDAdapt) and \LSQ Alg.~\ref{alg:corrector-lsq} (\pipelineNameSynLSQAdapt).

For pose estimation, we compare against baselines including CosyPose~\cite{Labbe20eccv-CosyPose}, BundleSDF~\cite{Wen23cvpr-bundlesdf} and GDRNet++~\cite{Wang21cvpr-GDRNetGeometryGuided, Liu22eccvw-gdrnppBOP}.
For shape reconstruction, we compare against Shap-E~\cite{Jun23arxiv-shap-e}, a state-of-the-art conditional diffusion generative model.
We report ADD-S metrics~\cite{Xiang17rss-posecnn} and the Chamfer distance between ground truth model and reconstructed shape (labeled as $e_{shape}$).
We also compute the area under curve (AUC) metric for both ADD-S and $e_{shape}$~\cite{Xiang17rss-posecnn}.
All methods use the same ground truth segmentation masks.

\begin{table}[t!]
    \footnotesize
    \centering
\begin{tabular}{llllll}\toprule
                   & \multicolumn{2}{c}{\small $e_{shape}$ $\downarrow$} & \multicolumn{3}{c}{\small $e_{shape}$ (AUC) $\uparrow$} \\ \cmidrule(lr){2-3} \cmidrule(lr){4-6}
Method                                  & \small Mean            & \small Median          & \small 3 cm          & \small 5 cm         & \small 10 cm       \\ \midrule
\footnotesize Shap-E~\cite{Jun23arxiv-shap-e}                           & 0.099           & 0.052           & 0.05        & 0.17        & 0.43       \\ 
\footnotesize \texttt{CRISP-Syn}               & 0.045           & 0.032           & 0.18         & 0.35        & 0.58       \\ 
\footnotesize \makecell[l]{\texttt{CRISP-Syn-ST} \\ \texttt{(LSQ)}} & \tabsecond 0.037           & \tabsecond 0.024           & \tabsecond 0.25         & \tabsecond 0.43        & \tabsecond 0.65       \\ 
\footnotesize \makecell[l]{\texttt{CRISP-Syn-ST} \\ \texttt{(BCD)}} & \tabthird 0.039           & \tabthird 0.028           & \tabthird 0.19        & \tabthird 0.39        & \tabthird 0.62       \\ 
\footnotesize \texttt{CRISP-Real}              & \tabfirst 0.026 & \tabfirst 0.016  & \tabfirst 0.40  & \tabfirst 0.58 & \tabfirst 0.75       \\ \bottomrule
\end{tabular}
\caption{Evaluation results on the $e_{shape}$ and $e_{shape}$ (AUC) metrics for the YCBV dataset.}
\label{tab:ycbv-shape-sota}
\vspace{-4mm}
\end{table}

\begin{table}[t!]
    \footnotesize
    \centering
\begin{tabular}{llllll}\toprule
                   & \multicolumn{2}{c}{ADD-S $\downarrow$} & \multicolumn{3}{c}{ADD-S (AUC) $\uparrow$} \\ \cmidrule(lr){2-3} \cmidrule(lr){4-6}
Method             & Mean        & Median      & 1 cm       & 2 cm      & 3 cm      \\\midrule
\footnotesize CosyPose~\cite{Labbe20eccv-CosyPose}    & \tabfirst 0.010 & \tabsecond 0.007                 & \tabsecond 0.30  & \tabsecond 0.56     & \tabsecond 0.68     \\
\footnotesize BundleSDF~\cite{Wen23cvpr-bundlesdf}        & 0.014           & 0.012                 & 0.14                & 0.37     & 0.55     \\
\footnotesize GDRNet++~\cite{Liu22eccvw-gdrnppBOP}         & \tabthird 0.013 & \tabthird 0.011                 & 0.22                & 0.43     & \tabthird 0.58     \\
\footnotesize \texttt{CRISP-Syn} & 0.020  & 0.015 & 0.12 & 0.30     & 0.45     \\
\footnotesize \makecell[l]{\texttt{CRISP-Syn-ST} \\ \texttt{(LSQ)}} & 0.016 & 0.010                 & \tabthird 0.24                & \tabthird 0.44     & 0.56     \\
\footnotesize \makecell[l]{\texttt{CRISP-Syn-ST} \\ \texttt{(BCD)}} & 0.017 & 0.010                 & 0.22                & 0.42     & 0.54     \\ 
\footnotesize \texttt{CRISP-Real} & \tabsecond 0.011  & \tabfirst 0.005       & \tabfirst 0.41      & \tabfirst 0.62     & \tabfirst 0.73     \\\bottomrule
\end{tabular}
\caption{Evaluation results on the ADD-S and ADD-S (AUC) metrics for the YCBV dataset.}
\label{tab:ycbv-adds-sota}
\vspace{-5mm}
\end{table}

\begin{figure}[t!]
    \vspace{1mm}
  \centering
  \includegraphics[width=0.9\linewidth]{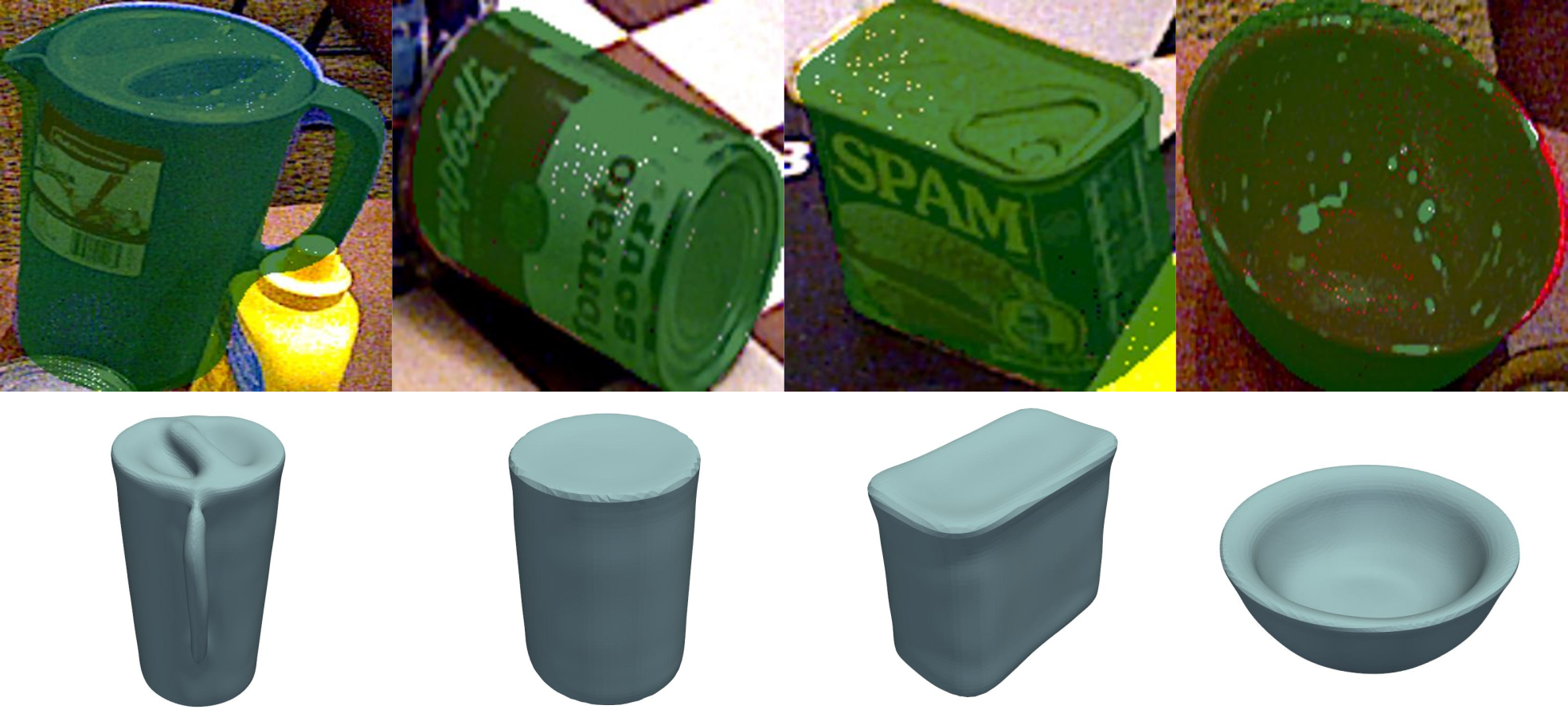}
  \caption{Qualitative examples of \pipelineName on the YCBV dataset. \emph{Top}: projection of transformed reconstructed mesh with our estimation. \emph{Bottom:} reconstructed mesh. See Appendix for more examples.}
  \label{fig:ycbv-qualitative}
\end{figure}

\emph{Results.} 
Tab.~\ref{tab:ycbv-shape-sota} reports shape reconstruction performances.
\pipelineNameReal performs the best, with the mean $e_{shape}$ of 0.026, significantly outperforming Shap-E ($0.099$) by 73\%.
In addition, self-training improves \pipelineNameSyn's performance for all metrics. 
Compared to \pipelineNameSyn, \pipelineNameSynLSQAdapt and \pipelineNameSynPGDAdapt improve on mean $e_{shape}$ by 17\% and 12\% and $e_{shape}$ (AUC) (5 cm) by 22\% and 11\% respectively.
\pipelineNameSynLSQAdapt slightly outperforms \pipelineNameSynPGDAdapt on the $e_{shape}$ (AUC) metric at 5 and 10 cm.
\pipelineNameReal also achieves the best performance at all ADD-S (AUC) metrics, 
with CosyPose trailing behind.
Our self-training approach improves \pipelineNameSyn's performance, 
with \pipelineNameSynLSQAdapt and \pipelineNameSynPGDAdapt improve on mean ADD-S by 20\% and 17\% respectively and ADD-S (AUC) (2 cm) by 44\% and 38\% respectively.
GDRNet++ achieves slightly better mean ADD-S then \pipelineNameReal ($0.010$ vs. $0.011$).
We show some qualitative examples of \pipelineName in Fig.~\ref{fig:ycbv-qualitative}.

\begin{figure}[t!]
\begin{subfigure}{.23\textwidth}
	\centering
  \includegraphics[width=\linewidth]{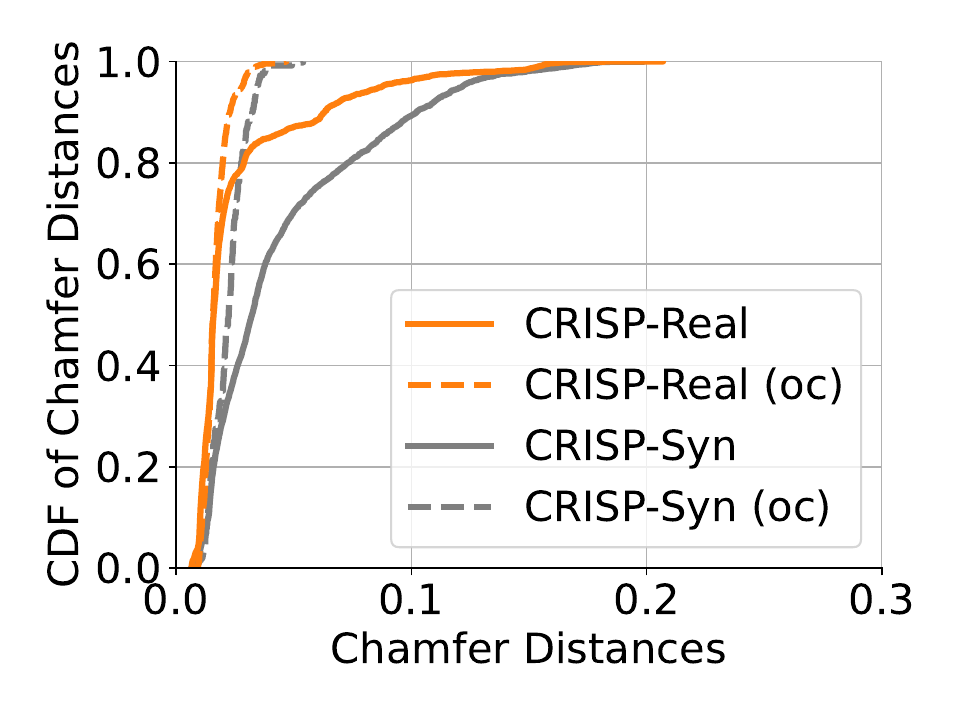}
	\caption{}	
	\label{fig:ycbv-oc-shape-cdf}
\end{subfigure}%
~
\begin{subfigure}{.23\textwidth}
  \centering
 	\includegraphics[width=\linewidth]{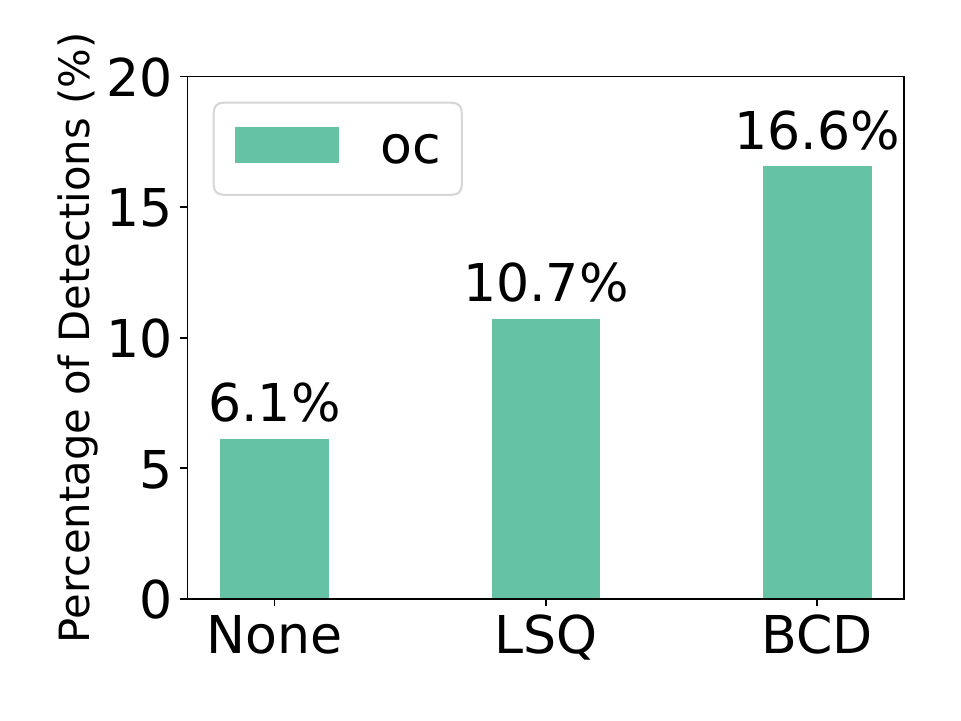}
  \caption{}
  \label{fig:ycbv-corrector-oc}
\end{subfigure}
\caption{(a) Cumulative distribution function of the Chamfer distances of the reconstructed shapes. Methods append with (\texttt{oc}) indicate the instances pass \texttt{oc}. 
(b) Comparison of percentages of test instances which pass \texttt{oc} (label: \texttt{oc}).
}
\vspace{-4mm}
\end{figure}

To understand why self-training works, Fig.~\ref{fig:ycbv-oc-shape-cdf} shows the CDFs of $e_{shape}$ of \pipelineName and \pipelineNameSyn with and without \texttt{oc} check applied.
With \texttt{oc}, the CDF curves are shifted to the far left,
indicating that the check 
filters out instances with high $e_{shape}$ without access to ground truth.
By collecting pseudo-labels using the \texttt{oc} check, we are collecting good samples that have low $e_{shape}$, enabling self-training. 
Fig.~\ref{fig:ycbv-corrector-oc} shows the impact of correctors in self-training. 
Without corrector, \pipelineNameSyn only 6.1\% of the outputs pass \texttt{oc}.
With LSQ (Alg.~\ref{alg:corrector-lsq}) and BCD (Alg.~\ref{alg:corrector}), the percentage increases to 10.7\% and 16.6\%, respectively.

\subsection{SPE3R Dataset}
\label{sec:expt-satellite}

\emph{Setup.} %
We evaluate \pipelineName on the SPE3R dataset~\cite{Park24aiaa-spe3r},
which contains photorealistic renderings of 64 real-world satellites.
We use the official train set, which contains 57 satellites (800 images per satellite).
Seven satellites are withheld for testing.
We use an Adam optimizer with a learning rate of $10^{-4}$, a batch size of 16 and for 100 epochs (label: \pipelineName).
In addition, we self-train \pipelineName for 10 epochs on the test set, using only the \PGD corrector (Alg.~\ref{alg:corrector}) due to the unknown test objects (label: \pipelineNameST).
Following~\cite{Park24aiaa-spe3r}, we report 
$L_1$ Chamfer distance between ground-truth model and reconstructed shape (label: $e^{L_1}_{shape}$) and the $L_1$ Chamfer distance between the ground-truth transformed model and the reconstructed model transformed using estimated pose (label: $e^{L_1}_{pose}$).

\begin{table}[t!]
	\small
\centering
 \begin{tabular}{lllll} 
	\toprule
 Method & \multicolumn{2}{c}{$e^{L_1}_{shape} \downarrow$} & \multicolumn{2}{c}{$e^{L_1}_{pose} \downarrow$} \\ \cmidrule(lr){2-3} \cmidrule(lr){4-5}
 			 & Mean & Median & Mean & Median \\ \midrule
 Shap-E~\cite{Jun23arxiv-shap-e} &0.183 	&0.171 	&\textendash	&\textendash 	\\ 
 SQRECON~\cite{Park24aiaa-spe3r} 	& \tabsecond 0.145 	& \tabsecond 0.134 	& \tabthird 0.304	& \tabthird 0.258 	\\ 
 \pipelineName  					& \tabthird 0.163	    & \tabthird 0.159	   & \tabsecond 0.292 	& \tabsecond 0.211	\\
 \pipelineNameST 	& \tabfirst 0.141 	& \tabfirst 0.125 	& \tabfirst 0.224 	& \tabfirst 0.168 \\
 \bottomrule
 \end{tabular}
 \caption{Evaluation of \pipelineName and \pipelineNameST on the SPE3R dataset.}
 \label{tab:spe3r-results}
 \vspace{-5mm}
\end{table}

\emph{Results.} %
As shown in Tab.~\ref{tab:spe3r-results},
\pipelineName outperforms SQRECON and Shap-E in terms of mean and median pose error $e^{L_1}_{pose}$, whereas SQRECON achieves better $e^{L_1}_{shape}$.
With self-training, \pipelineNameST outperforms SQRECON on all metrics, improving mean $e^{L_1}_{shape}$ by 13\% and $e^{L_1}_{pose}$ by 23\%.  

Fig.~\ref{fig:sat-shape-adaptation} shows a qualitative example of self-training improving shape reconstruction performance. This satellite, \texttt{smart-1}, is unseen during training.
Initially \pipelineName recalls the wrong shape, similar to the \texttt{integral} satellite in the train set (middle of Fig.~\ref{fig:sat-shape-adaptation}).
After self-training, the model is able to reconstruct a model (bottom of Fig.~\ref{fig:sat-shape-adaptation}) much closer to the ground truth model (top of Fig.~\ref{fig:sat-shape-adaptation}).

\begin{figure}[t]
\centering
  \includegraphics[width=0.9\linewidth]{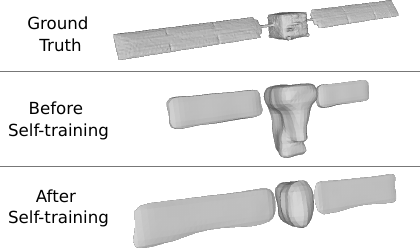}
\caption{Qualitative example of self-training improving shape reconstruction.}
\label{fig:sat-shape-adaptation}
\vspace{-6mm}
\end{figure}

	\subsection{NOCS Dataset}
\label{sec:expt-nocs}

\begin{table}[t!]
    \centering
    \small
    \begin{tabular}{llllll}
    \toprule
    & \multicolumn{5}{c}{$e_{shape}$ (mm) $\downarrow$ } \\\cmidrule(lr){2-6}
    Methods       & \small Bottle & \small Camera  & \small Laptop & \small Mug  & \small Avg  \\
    \midrule
    SPD~\cite{Tian20eccv-SPD}   & 3.44             & 8.89              & 2.91             & 1.02             & 3.17 \\
    SGPA~\cite{Chen21iccv-SGPA}          & 2.93             & \tabthird 5.51    & \tabthird 1.62   & 1.12             & 2.44 \\
    CASS~\cite{Chen20iccv-learningCanonicalShape}      & \tabsecond 0.75  & \tabsecond 0.77   & 3.73              & \tabsecond 0.32 & \tabsecond 1.06 \\
    RePoNet~\cite{Ze22neurips-wild6d}       & \tabthird 1.51   & 8.79              & \tabsecond 1.01   & \tabthird 0.94  & \tabthird 2.37 \\
    \pipelineName & \tabfirst 0.55   &  \tabfirst 0.37   &  \tabfirst 0.51   &  \tabfirst 0.16 &  \tabfirst 0.35 \\ \bottomrule
    \end{tabular}
    \caption{Shape reconstruction results on the \NOCSREALTwoSevenFive dataset. Tabulates average Chamfer distance (mm) for four categories and the average distance for all categories (full table in Appendix~\ref{sec:app:add-exp}).}
    \label{tab:nocs-shape}
    \vspace{-4mm}
    \end{table}

\emph{Setup.} %
    We evaluate \pipelineName against baselines on the \NOCS dataset~\cite{Wang19-normalizedCoordinate}. 
    The \NOCS dataset is divided into two subsets: \NOCSCAMERA, a synthetic dataset with 300K rendered images of ShapeNet objects,
    and \NOCSREAL with 8K RGB-D frames of 18 different 
    real-world scenes.%
    We train \pipelineName on the train sets of \NOCSCAMERA and \NOCSREAL for 50 epochs, using an Adam optimizer with a learning rate of $10^{-4}$ and a
     cosine annealing scheduler with restarts per 4K iterations.
    For evaluation, we report results on the \NOCSREALTwoSevenFive,
    which is the test set of \NOCSREAL containing 2.75K images.
    We report $e_{shape}$, the average Chamfer distance between the reconstructed shape and the ground truth CAD model, and mean average precision (mAP) at different thresholds, following the metrics proposed by~\cite{Wang19-normalizedCoordinate}.

\begin{table}[t!]
    \centering
    \small
    \begin{tabular}{llllll}
        \toprule
       & &  \multicolumn{4}{c}{mAP $\uparrow$ } \\\cmidrule(lr){3-6}
    &Methods          & $IoU_{50}$ & $IoU_{75}$ & \makecell{\scriptsize $5^{\circ}$ \\ $5 \mathrm{~cm}$}& \makecell{\scriptsize $10^{\circ}$ \\ $5 \mathrm{~cm}$}\\ \midrule
    \parbox[t]{2mm}{\multirow{7}{*}{\rotatebox[origin=c]{90}{Category-level}}} &NOCS~\cite{Wang19-normalizedCoordinate}      & 78.0            & 30.1           & 10.0            & 25.2        \\
    &Metric Scale~\cite{Lee21ral-metricScale}     & 68.1            & \textendash    & 5.3             & 24.7        \\
    &SPD~\cite{Tian20eccv-SPD}       & 77.3            & 53.2           & 21.4            & 54.1        \\
    &CASS~\cite{Chen20iccv-learningCanonicalShape}     & 77.7            & \textendash    & 23.5            & 58.0          \\
    &SGPA~\cite{Chen21iccv-SGPA}      & \tabthird 80.1  & \tabsecond 61.9& \tabsecond 39.6 & \tabfirst 70.7        \\
    &RePoNet~\cite{Ze22neurips-wild6d}          & \tabsecond 81.1 & \textendash    & \tabfirst 40.4  & \tabsecond 68.8 \\
    &SSC-6D~\cite{Peng22aaai-selfSupPoseShape}  & 73.0            & \textendash    & 19.6            & 54.5        \\ \midrule
    \parbox[t]{5mm}{\multirow{3}{*}{\rotatebox[origin=c]{90}{\makecell{Category\\-agnostic}}}} &DualPoseNet~\cite{Lin21iccv-dualposenet}    & 76.1            & 55.2           & \tabfirst 31.3  & 60.4        \\
    &FSD~\cite{Lunayach24icra-FSD}      & 77.4            & \tabsecond 61.9& 28.1            & 61.5        \\
    &\pipelineName    & \tabfirst 83.5  & \tabfirst 70.5 & 22.5            & \tabthird 62.8       \\ \bottomrule
    \end{tabular}
    \caption{Pose estimation results on the \NOCSREALTwoSevenFive dataset. We show here the average mAP at different thresholds for 3D IoU (IoU$_{50}$ and IoU$_{75}$) and rotation and translation errors at $5^{\circ} ~5 \mathrm{~cm}$ and $10^{\circ} ~5 \mathrm{~cm}$ thresholds (full table in Appendix~\ref{sec:app:add-exp}).}
    \label{tab:nocs-pose}
    \vspace{-4mm}
    \end{table}

\emph{Results.} %
Tab.~\ref{tab:nocs-shape} shows the shape reconstruction error 
for various baselines.
\pipelineName outperforms the baselines in all categories, achieving a 67\% improvement over the next best baseline, CASS. %
In the challenging object category of camera, \pipelineName achieves a 95\% reduction comparing to SPD~\cite{Tian20eccv-SPD}, and a 52\% reduction comparing to CASS.  
Tab.~\ref{tab:nocs-pose} reports the pose estimation errors of methods tested. %
We separate out category-level methods from category-agnostic ones.
The former either train one model per category, use explicit shape priors, or hard-code the number of categories in the network architecture.
The latter train one model for all categories and do not hard code the number of categories in the architecture, which are more scalable.
\pipelineName achieves the best performance in terms of mAP for $IoU_{50}$ and $IoU_{75}$, 
while ranked third in terms of mAP at $10^{\circ}~5 \mathrm{~cm}$ threshold.
We note that \pipelineName's rotation error is higher than category-level baselines such as RePoNet~\cite{Ze22neurips-wild6d} and SGPA~\cite{Chen21iccv-SGPA}.

	\subsection{Ablation Studies}
\label{sec:expt-ablations}

\begin{table}[t!]
\small
\centering
\begin{tabular}{ccc}
\toprule
Ablations & ADD-S (AUC) $\uparrow$ & $e_{shape}$ (AUC) $\uparrow$ \\
& \multicolumn{1}{c}{2 cm} & \multicolumn{1}{c}{5 cm}\\ \midrule
\makecell[c]{No Projection  $\Pi$ \\~(Line~\ref{algo:shape-update}, \PGD)}  & 0.27    &  0.02   \\ 
Single-View                & \tabsecond 0.37  & 0.29 \\
No Corrector               & \tabthird 0.30  & \tabthird 0.35 \\
\makecell[c]{Direct \\Self-training Loss}  & 0.29  & \tabsecond 0.38  \\ \midrule
Proposed                   & \tabfirst 0.42  & \tabfirst 0.39  \\\bottomrule
\end{tabular}
\caption{Ablation studies on test-time adaptation for \pipelineName}
\label{tab:ablations}
\vspace{-5mm}
\end{table}

We ablate a few design decisions to see their effects on self-training performance (Tab.~\ref{tab:ablations}).
\begin{wraptable}{r}{0pt}
    \footnotesize
    \begin{tabular}{llll}
    \toprule
              & \multicolumn{3}{c}{Runtime (ms) $\downarrow$} \\ \cmidrule{2-4}
    Component & Mean     & Median     & SD       \\ \midrule
    \pipelineName  & 125      & 127        & 161      \\
    LSQ       & 251      & 214        & 142      \\
    BCD (25)  & 3824     & 3483       & 1667     \\
    BCD (50)  & 7630     & 6942       & 3353     \\ \bottomrule
    \end{tabular}
    \caption{Runtime comparisons of \pipelineName, LSQ (Alg.~\ref{alg:corrector-lsq}) and BCD (Alg.~\ref{alg:corrector}) with 25 and 50 maximum iterations.}
    \label{tab:crisp-runtime}
\end{wraptable}
Without the projection in Line~\ref{algo:shape-update}, in the \PGD Algorithm~\ref{alg:corrector}, $e_{shape}$ drops significantly (Row 1, Tab.~\ref{tab:ablations}). 
Restricting the corrector to single-view (Row 2, Tab.~\ref{tab:ablations})
and self-train without the corrector by directly certifying and self-training \pipelineName outputs (Row 3, Tab.~\ref{tab:ablations}) adversely affects both shape and pose estimation performance.
Lastly, self-training without any pseudo-labels, but with a direct self-supervision loss similar to~\cite{Peng22aaai-selfSupPoseShape} achieves similar $e_{shape}$ but worse ADD-S (AUC) metric (Row 4, Tab.~\ref{tab:ablations}).
Tab.~\ref{tab:crisp-runtime} shows the runtime of \pipelineName and correctors.
LSQ (Alg.~\ref{alg:corrector-lsq}) achieves significantly faster runtime than BCD (Alg.~\ref{alg:corrector}).

\section{Limitations and Future Work}
\label{sec:limitations}
We currently assume ground-truth object data association during test-time.
A future direction is to investigate how tracking errors affect self-training and how to mitigate them.
Another limitation is the use of latent shape codes for projection in the corrector, which scales linearly with the number of objects learned.
One potential solution is to use dimensionality reduction techniques such as PCA.
This will open up opportunities to self-train \pipelineName on larger datasets,
potentially enabling it as a foundation model for object pose and shape estimation that can perform test-time adaptation.

\section{Conclusion}
\label{sec:conclusion}
We introduce \pipelineName, a category-agnostic object pose and shape estimation pipeline, and a self-training method \pipelineNameST to bridge potential domain gaps.
\pipelineName is fast at inference, amenable for multi-view pose and shape estimation and can be incorporated into real-time robotic perception pipeline for downstream tasks such as manipulation and object-SLAM. 
Experiments on the YCBV, SPE3R, and NOCS datasets show that \pipelineName is competitive in pose estimation and outperforms baselines in shape reconstruction.

{
\small
\bibliographystyle{ieeenat_fullname}
\bibliography{myRefs,../../references/refs}

\begin{thebibliography}{48}
\providecommand{\natexlab}[1]{#1}
\providecommand{\url}[1]{\texttt{#1}}
\expandafter\ifx\csname urlstyle\endcsname\relax
  \providecommand{\doi}[1]{doi: #1}\else
  \providecommand{\doi}{doi: \begingroup \urlstyle{rm}\Url}\fi

\bibitem[Arun et~al.(1987)Arun, Huang, and Blostein]{Arun87pami}
K.S. Arun, T.S. Huang, and S.D. Blostein.
\newblock Least-squares fitting of two 3-{D} point sets.
\newblock \emph{{IEEE} Trans. Pattern Anal. Machine Intell.}, 9\penalty0
  (5):\penalty0 698--700, 1987.

\bibitem[Brazil et~al.(2023)Brazil, Kumar, Straub, Ravi, Johnson, and
  Gkioxari]{Brazil23cvpr-omni3d}
Garrick Brazil, Abhinav Kumar, Julian Straub, Nikhila Ravi, Justin Johnson, and
  Georgia Gkioxari.
\newblock Omni3d: A large benchmark and model for 3d object detection in the
  wild.
\newblock In \emph{IEEE Conf. on Computer Vision and Pattern Recognition
  (CVPR)}, pages 13154--13164, 2023.

\bibitem[Chan et~al.(2021)Chan, Monteiro, Kellnhofer, Wu, and
  Wetzstein]{Chan21cvpr-PiGAN}
Eric~R Chan, Marco Monteiro, Petr Kellnhofer, Jiajun Wu, and Gordon Wetzstein.
\newblock pi-gan: Periodic implicit generative adversarial networks for
  3d-aware image synthesis.
\newblock In \emph{IEEE Conf. on Computer Vision and Pattern Recognition
  (CVPR)}, pages 5799--5809, 2021.

\bibitem[Chen et~al.(2019)Chen, Cao, Parra, and
  Chin]{Chen19ICCVW-satellitePoseEstimation}
Bo Chen, Jiewei Cao, Alvaro Parra, and Tat-Jun Chin.
\newblock Satellite pose estimation with deep landmark regression and nonlinear
  pose refinement.
\newblock In \emph{2019 IEEE/CVF International Conference on Computer Vision
  Workshop (ICCVW)}, pages 2816--2824. IEEE, 2019.

\bibitem[Chen et~al.(2020)Chen, Li, Wang, and
  Xu]{Chen20iccv-learningCanonicalShape}
Dengsheng Chen, Jun Li, Zheng Wang, and Kai Xu.
\newblock Learning canonical shape space for category-level 6d object pose and
  size estimation.
\newblock In \emph{IEEE Conf. on Computer Vision and Pattern Recognition
  (CVPR)}, pages 11973--11982, 2020.

\bibitem[Chen and Dou(2021)]{Chen21iccv-SGPA}
Kai Chen and Qi Dou.
\newblock Sgpa: Structure-guided prior adaptation for category-level 6d object
  pose estimation.
\newblock In \emph{Intl. Conf. on Computer Vision (ICCV)}, pages 2773--2782,
  2021.

\bibitem[Deng et~al.(2024)Deng, Lu, and Zhang]{Deng24cvpr-unsupervised}
Jiacheng Deng, Jiahao Lu, and Tianzhu Zhang.
\newblock Unsupervised template-assisted point cloud shape correspondence
  network.
\newblock In \emph{IEEE Conf. on Computer Vision and Pattern Recognition
  (CVPR)}, pages 5250--5259, 2024.

\bibitem[Denninger et~al.(2020)Denninger, Sundermeyer, Winkelbauer, Olefir,
  Hodan, Zidan, Elbadrawy, Knauer, Katam, and
  Lodhi]{Denninger20rss-blenderproc}
Maximilian Denninger, Martin Sundermeyer, Dominik Winkelbauer, Dmitry Olefir,
  Tomas Hodan, Youssef Zidan, Mohamad Elbadrawy, Markus Knauer, Harinandan
  Katam, and Ahsan Lodhi.
\newblock Blenderproc: Reducing the reality gap with photorealistic rendering.
\newblock In \emph{16th Robotics: Science and Systems, RSS 2020, Workshops},
  2020.

\bibitem[Diamond and Boyd(2016)]{Diamond16cvxpy}
Steven Diamond and Stephen Boyd.
\newblock {CVXPY}: {A} {P}ython-embedded modeling language for convex
  optimization.
\newblock \emph{Journal of Machine Learning Research}, 17\penalty0
  (83):\penalty0 1--5, 2016.

\bibitem[H{\"a}gele et~al.(2016)H{\"a}gele, Nilsson, Pires, and
  Bischoff]{Hagele16springer-industrialRobotics}
Martin H{\"a}gele, Klas Nilsson, J~Norberto Pires, and Rainer Bischoff.
\newblock Industrial robotics.
\newblock \emph{Springer handbook of robotics}, pages 1385--1422, 2016.

\bibitem[Hoda{\v{n}} et~al.(2020)Hoda{\v{n}}, Sundermeyer, Drost, Labb{\'e},
  Brachmann, Michel, Rother, and Matas]{Hodan20eccvw-BOPChallenge}
Tom{\'a}{\v{s}} Hoda{\v{n}}, Martin Sundermeyer, Bertram Drost, Yann Labb{\'e},
  Eric Brachmann, Frank Michel, Carsten Rother, and Ji{\v{r}}{\'i} Matas.
\newblock {BOP} challenge 2020 on {6D} object localization.
\newblock \emph{European Conference on Computer Vision Workshops (ECCVW)},
  2020.

\bibitem[Irshad et~al.(2022)Irshad, Zakharov, Ambrus, Kollar, Kira, and
  Gaidon]{Irshad22eccv-shapo}
Muhammad~Zubair Irshad, Sergey Zakharov, Rares Ambrus, Thomas Kollar, Zsolt
  Kira, and Adrien Gaidon.
\newblock Shapo: Implicit representations for multi-object shape, appearance,
  and pose optimization.
\newblock In \emph{European Conf. on Computer Vision (ECCV)}, pages 275--292.
  Springer, 2022.

\bibitem[Jun and Nichol(2023{\natexlab{a}})]{Jun23arxiv-shap-e}
Heewoo Jun and Alex Nichol.
\newblock Shap-e: Generating conditional 3d implicit functions.
\newblock \emph{arXiv preprint arXiv:2305.02463}, 2023{\natexlab{a}}.

\bibitem[Jun and Nichol(2023{\natexlab{b}})]{Jun23arxiv-shape}
Heewoo Jun and Alex Nichol.
\newblock Shap-e: Generating conditional 3d implicit functions.
\newblock \emph{arXiv preprint arXiv:2305.02463}, 2023{\natexlab{b}}.

\bibitem[{Labbe} et~al.(2020){Labbe}, {Carpentier}, {Aubry}, and
  {Sivic}]{Labbe20eccv-CosyPose}
Y. {Labbe}, J. {Carpentier}, M. {Aubry}, and J. {Sivic}.
\newblock {CosyPose}: Consistent multi-view multi-object {6D} pose estimation.
\newblock In \emph{European Conf. on Computer Vision (ECCV)}, 2020.

\bibitem[Labb{\'e} et~al.(2022)Labb{\'e}, Manuelli, Mousavian, Tyree,
  Birchfield, Tremblay, Carpentier, Aubry, Fox, and
  Sivic]{Labbe22corl-megapose}
Yann Labb{\'e}, Lucas Manuelli, Arsalan Mousavian, Stephen Tyree, Stan
  Birchfield, Jonathan Tremblay, Justin Carpentier, Mathieu Aubry, Dieter Fox,
  and Josef Sivic.
\newblock Megapose: 6d pose estimation of novel objects via render \& compare.
\newblock 2022.

\bibitem[Lee et~al.(2021)Lee, Lee, Kim, and Kweon]{Lee21ral-metricScale}
Taeyeop Lee, Byeong-Uk Lee, Myungchul Kim, and In~So Kweon.
\newblock Category-level metric scale object shape and pose estimation.
\newblock \emph{{IEEE} Robotics and Automation Letters}, 6\penalty0
  (4):\penalty0 8575--8582, 2021.

\bibitem[Li et~al.(2018)Li, Wang, Ji, Xiang, and Fox]{Li18eccv-DeepIMDeep}
Yi Li, Gu Wang, Xiangyang Ji, Yu Xiang, and Dieter Fox.
\newblock {{DeepIM}}: {{Deep Iterative Matching}} for {{6D Pose Estimation}}.
\newblock In \emph{European Conf. on Computer Vision (ECCV)}, pages 683--698,
  2018.

\bibitem[Lin et~al.(2021)Lin, Wei, Li, Xu, Jia, and Li]{Lin21iccv-dualposenet}
Jiehong Lin, Zewei Wei, Zhihao Li, Songcen Xu, Kui Jia, and Yuanqing Li.
\newblock Dualposenet: Category-level 6d object pose and size estimation using
  dual pose network with refined learning of pose consistency.
\newblock In \emph{Intl. Conf. on Computer Vision (ICCV)}, pages 3560--3569,
  2021.

\bibitem[Liu et~al.(2024)Liu, Xu, Jin, Chen, Varma~T, Xu, and
  Su]{Liu24neuips-oneTwoThree}
Minghua Liu, Chao Xu, Haian Jin, Linghao Chen, Mukund Varma~T, Zexiang Xu, and
  Hao Su.
\newblock One-2-3-45: Any single image to 3d mesh in 45 seconds without
  per-shape optimization.
\newblock \emph{Advances in Neural Information Processing Systems (NIPS)}, 36,
  2024.

\bibitem[Liu et~al.(2022)Liu, Zhang, Zhang, Fu, Tang, Liang, Tang, Cheng,
  Zhang, Wang, and Ji]{Liu22eccvw-gdrnppBOP}
Xingyu Liu, Ruida Zhang, Chenyangguang Zhang, Bowen Fu, Jiwen Tang, Xiquan
  Liang, Jingyi Tang, Xiaotian Cheng, Yukang Zhang, Gu Wang, and Xiangyang Ji.
\newblock Gdrnpp.
\newblock \url{https://github.com/shanice-l/gdrnpp_bop2022}, 2022.

\bibitem[Lunayach et~al.(2024)Lunayach, Zakharov, Chen, Ambrus, Kira, and
  Irshad]{Lunayach24icra-FSD}
Mayank Lunayach, Sergey Zakharov, Dian Chen, Rares Ambrus, Zsolt Kira, and
  Muhammad~Zubair Irshad.
\newblock Fsd: Fast self-supervised single rgb-d to categorical 3d objects.
\newblock In \emph{IEEE Intl. Conf. on Robotics and Automation (ICRA)}, pages
  14630--14637. IEEE, 2024.

\bibitem[Mousavian et~al.(2017)Mousavian, Anguelov, Flynn, and
  Kosecka]{Mousavian17cvpr-3dBbox}
Arsalan Mousavian, Dragomir Anguelov, John Flynn, and Jana Kosecka.
\newblock 3d bounding box estimation using deep learning and geometry.
\newblock In \emph{IEEE Conf. on Computer Vision and Pattern Recognition
  (CVPR)}, pages 7074--7082, 2017.

\bibitem[Newell et~al.(2016)Newell, Yang, and Deng]{Newell16-stackedHourglass}
Alejandro Newell, Kaiyu Yang, and Jia Deng.
\newblock Stacked hourglass networks for human pose estimation.
\newblock In \emph{European Conf. on Computer Vision (ECCV)}, pages 483--499.
  Springer, 2016.

\bibitem[Oquab et~al.(2023)Oquab, Darcet, Moutakanni, Vo, Szafraniec, Khalidov,
  Fernandez, Haziza, Massa, El-Nouby, et~al.]{Oquab23arxiv-dinov2}
Maxime Oquab, Timoth{\'e}e Darcet, Th{\'e}o Moutakanni, Huy Vo, Marc
  Szafraniec, Vasil Khalidov, Pierre Fernandez, Daniel Haziza, Francisco Massa,
  Alaaeldin El-Nouby, et~al.
\newblock Dinov2: Learning robust visual features without supervision.
\newblock \emph{arXiv preprint arXiv:2304.07193}, 2023.

\bibitem[Park et~al.(2019)Park, Florence, Straub, Newcombe, and
  Lovegrove]{Park19cvpr-deepSDF}
J.J. Park, P. Florence, J. Straub, R. Newcombe, and S. Lovegrove.
\newblock {DeepSDF}: Learning continuous signed distance functions for shape
  representation.
\newblock In \emph{IEEE Conf. on Computer Vision and Pattern Recognition
  (CVPR)}. IEEE, 2019.

\bibitem[Park and D'Amico(2024)]{Park24aiaa-spe3r}
Tae~Ha Park and Simone D'Amico.
\newblock Rapid abstraction of spacecraft 3d structure from single 2d image.
\newblock In \emph{AIAA SCITECH 2024 Forum}, page 2768, 2024.

\bibitem[Paszke et~al.(2019)Paszke, Gross, Massa, Lerer, Bradbury, Chanan,
  Killeen, Lin, Gimelshein, Antiga, et~al.]{Paszke19neurips-pytorch}
Adam Paszke, Sam Gross, Francisco Massa, Adam Lerer, James Bradbury, Gregory
  Chanan, Trevor Killeen, Zeming Lin, Natalia Gimelshein, Luca Antiga, et~al.
\newblock Pytorch: An imperative style, high-performance deep learning library.
\newblock \emph{Advances in neural information processing systems}, 32, 2019.

\bibitem[Pavlakos et~al.(2017)Pavlakos, Zhou, Chan, Derpanis, and
  Daniilidis]{Pavlakos17icra-semanticKeypoints}
G. Pavlakos, X. Zhou, A. Chan, K. Derpanis, and K. Daniilidis.
\newblock 6-dof object pose from semantic keypoints.
\newblock In \emph{IEEE Intl. Conf. on Robotics and Automation (ICRA)}, 2017.

\bibitem[Peng et~al.(2022)Peng, Yan, Wen, and Sun]{Peng22aaai-selfSupPoseShape}
Wanli Peng, Jianhang Yan, Hongtao Wen, and Yi Sun.
\newblock Self-supervised category-level 6d object pose estimation with deep
  implicit shape representation.
\newblock In \emph{Nat. Conf. on Artificial Intelligence (AAAI)}, pages
  2082--2090, 2022.

\bibitem[Perez et~al.(2018)Perez, Strub, de~Vries, Dumoulin, and
  Courville]{Perez18aaai-FiLMVisual}
Ethan Perez, Florian Strub, Harm de Vries, Vincent Dumoulin, and Aaron
  Courville.
\newblock {{FiLM}}: {{Visual Reasoning}} with a {{General Conditioning Layer}}.
\newblock \emph{Nat. Conf. on Artificial Intelligence (AAAI)}, 32\penalty0 (1),
  2018.

\bibitem[Qi et~al.(2017)Qi, Su, Mo, and Guibas]{Qi17cvpr-pointnet}
Charles~R Qi, Hao Su, Kaichun Mo, and Leonidas~J Guibas.
\newblock Pointnet: Deep learning on point sets for {3D} classification and
  segmentation.
\newblock In \emph{IEEE Conf. on Computer Vision and Pattern Recognition
  (CVPR)}, pages 652--660, 2017.

\bibitem[Ranftl et~al.(2021)Ranftl, Bochkovskiy, and Koltun]{Ranftl21iccv-DPT}
Ren{\'e} Ranftl, Alexey Bochkovskiy, and Vladlen Koltun.
\newblock Vision transformers for dense prediction.
\newblock In \emph{Intl. Conf. on Computer Vision (ICCV)}, pages 12179--12188,
  2021.

\bibitem[Sanneman et~al.(2021)Sanneman, Fourie, Shah,
  et~al.]{Sanneman21ftr-stateIndustrialRobot}
Lindsay Sanneman, Christopher Fourie, Julie~A Shah, et~al.
\newblock The state of industrial robotics: Emerging technologies, challenges,
  and key research directions.
\newblock \emph{Foundations and Trends{\textregistered} in Robotics},
  8\penalty0 (3):\penalty0 225--306, 2021.

\bibitem[Shalev-Shwartz et~al.(2006)Shalev-Shwartz, Singer, Bennett, and
  Parrado-Hern{\'a}ndez]{Shalev06jmlr-simplexProjection}
Shai Shalev-Shwartz, Yoram Singer, Kristin~P Bennett, and Emilio
  Parrado-Hern{\'a}ndez.
\newblock Efficient learning of label ranking by soft projections onto
  polyhedra.
\newblock \emph{J. of Machine Learning Research}, 7\penalty0 (7), 2006.

\bibitem[Shi et~al.(2023)Shi, Talak, Maggio, and Carlone]{Shi23rss-ensemble}
J. Shi, R. Talak, D. Maggio, and L. Carlone.
\newblock A correct-and-certify approach to self-supervise object pose
  estimators via ensemble self-training.
\newblock 2023.

\bibitem[Sitzmann et~al.()Sitzmann, Martel, Bergman, Lindell, and
  Wetzstein]{Sitzmann20neurips-siren}
Vincent Sitzmann, Julien Martel, Alexander Bergman, David Lindell, and Gordon
  Wetzstein.
\newblock Implicit neural representations with periodic activation functions.
\newblock \emph{Advances in Neural Information Processing Systems (NIPS)},
  33:\penalty0 7462--7473.

\bibitem[Talak et~al.(2023)Talak, Peng, and Carlone]{Talak23tro-c3po}
R. Talak, L. Peng, and L. Carlone.
\newblock Certifiable {3D} object pose estimation: Foundations, learning
  models, and self-training.
\newblock \emph{{IEEE} Trans. Robotics}, 39\penalty0 (4):\penalty0 2805--2824,
  2023.

\bibitem[Tian et~al.(2020)Tian, Ang, and Lee]{Tian20eccv-SPD}
Meng Tian, Marcelo~H Ang, and Gim~Hee Lee.
\newblock Shape prior deformation for categorical 6d object pose and size
  estimation.
\newblock In \emph{European Conf. on Computer Vision (ECCV)}, pages 530--546.
  Springer, 2020.

\bibitem[Wang et~al.(2020)Wang, Manhardt, Shao, Ji, Navab, and
  Tombari]{Wang20eccv-Self6DSelfSupervised}
Gu Wang, Fabian Manhardt, Jianzhun Shao, Xiangyang Ji, Nassir Navab, and
  Federico Tombari.
\newblock {Self6D}: Self-supervised monocular {6D} object pose estimation.
\newblock In \emph{European Conf. on Computer Vision (ECCV)}, pages 108--125,
  2020.

\bibitem[Wang et~al.(2021)Wang, Manhardt, Tombari, and
  Ji]{Wang21cvpr-GDRNetGeometryGuided}
Gu Wang, Fabian Manhardt, Federico Tombari, and Xiangyang Ji.
\newblock {{GDR-Net}}: {{Geometry-Guided Direct Regression Network}} for
  {{Monocular 6D Object Pose Estimation}}.
\newblock In \emph{IEEE Conf. on Computer Vision and Pattern Recognition
  (CVPR)}, pages 16611--16621, 2021.

\bibitem[Wang et~al.(2019)Wang, Sridhar, Huang, Valentin, Song, and
  Guibas]{Wang19-normalizedCoordinate}
H. Wang, S. Sridhar, J. Huang, J. Valentin, S. Song, and L. Guibas.
\newblock Normalized object coordinate space for category-level 6d object pose
  and size estimation.
\newblock In \emph{IEEE Conf. on Computer Vision and Pattern Recognition
  (CVPR)}, pages 2642--2651, 2019.

\bibitem[Wang and Solomon(2019)]{Wang19nips-prnet}
Yue Wang and Justin~M Solomon.
\newblock {PRNet: Self-Supervised Learning for Partial-to-Partial
  Registration}.
\newblock In \emph{Advances in Neural Information Processing Systems (NIPS)},
  pages 8812--8824, 2019.

\bibitem[Wen et~al.(2023)Wen, Tremblay, Blukis, Tyree, M{\"u}ller, Evans, Fox,
  Kautz, and Birchfield]{Wen23cvpr-bundlesdf}
Bowen Wen, Jonathan Tremblay, Valts Blukis, Stephen Tyree, Thomas M{\"u}ller,
  Alex Evans, Dieter Fox, Jan Kautz, and Stan Birchfield.
\newblock Bundlesdf: Neural 6-dof tracking and 3d reconstruction of unknown
  objects.
\newblock In \emph{IEEE Conf. on Computer Vision and Pattern Recognition
  (CVPR)}, pages 606--617, 2023.

\bibitem[Xiang et~al.(2018)Xiang, Schmidt, Narayanan, and
  Fox]{Xiang17rss-posecnn}
Yu Xiang, Tanner Schmidt, Venkatraman Narayanan, and Dieter Fox.
\newblock {PoseCNN}: A convolutional neural network for {6D} object pose
  estimation in cluttered scenes.
\newblock In \emph{Robotics: Science and Systems (RSS)}, 2018.

\bibitem[Xiao and Snoek(2024)]{Xiao24arxiv-beyondAdaptation}
Zehao Xiao and Cees~GM Snoek.
\newblock Beyond model adaptation at test time: A survey.
\newblock \emph{arXiv preprint arXiv:2411.03687}, 2024.

\bibitem[Ze and Wang(2022)]{Ze22neurips-wild6d}
Yanjie Ze and Xiaolong Wang.
\newblock Category-level 6d object pose estimation in the wild: A
  semi-supervised learning approach and a new dataset.
\newblock \emph{Advances in Neural Information Processing Systems (NIPS)},
  35:\penalty0 27469--27483, 2022.

\bibitem[Zhao et~al.(2024)Zhao, Kwon, Streli, Pollefeys, and
  Holz]{Zhao24arxiv-egopressure}
Yiming Zhao, Taein Kwon, Paul Streli, Marc Pollefeys, and Christian Holz.
\newblock {EgoPressure}: A dataset for hand pressure and pose estimation in
  egocentric vision.
\newblock In \emph{ArXiv Preprint: 2409.02224}, 2024.

\end{thebibliography}
}

\newpage
\appendix

\section{Proof of Equivalence}
\label{sec:app:opt}
We show that the two optimization problems~\eqref{eq:corrector-01} and~\eqref{eq:corrector-02} are equivalent. 
\begin{lemma}
	If $p^\ast$ and $f^{\ast}$ are the optimal values of the problems~\eqref{eq:corrector-01} and~\eqref{eq:corrector-02}, respectively, then 
	\begin{equation}
		p^\ast = f^\ast.
	\end{equation}
	Let $(\MZ^\ast, \vh^\ast, \MR^\ast, \vt^\ast)$ be the optimal solution to~\eqref{eq:corrector-02}, and $\Delta = \MR^\ast \MX + \vt^{\ast}\ones_n\tran - \MZ^{\ast}$, then  $(\MZ^\ast + \Delta, \vh^\ast,  \MR^\ast, \vt^\ast)$ is the optimal solution to~\eqref{eq:corrector-01}.
\end{lemma}
\begin{proof}
	First, note that 
	the problem~\eqref{eq:corrector-02} is obtained from~\eqref{eq:corrector-01} by relaxing the hard equality constraint (\ie $\vz_i = \MR \vxx_i + \vt$). Therefore, 
	\begin{equation}
		f^{\ast} \leq p^\ast.
	\end{equation}
	Second, we are given that $(\MZ^\ast, \vh^\ast, \MR^\ast, \vt^\ast)$ is the optimal solution to~\eqref{eq:corrector-02}. Let  $\delta_i = \MR^\ast \vxx_i + \vt^{\ast} - \vz^{\ast}_i$. Then, 
	\begin{equation}
		\tilde{\vz}_i = \vz^{\ast}_i + \delta_i = \MR^\ast \vxx_i + \vt^{\ast},
	\end{equation}
	satisfies the equality constraint in~\eqref{eq:corrector-01}.
	Furthermore, 
	\begin{equation}
		f^\ast = \sum_{i=1}^{n} |  f_d(\MR^\ast \vxx_i + \vt^\ast~|~\vh^\ast) |^2 =  \sum_{i=1}^{n} |  f_d(\tilde{\vz}_i~|~\vh^\ast) |^2,
 	\end{equation}
 	which implies that a feasible solution $(\tilde{\MZ}, \vh^\ast, \MR^\ast, \vt^\ast)$ in~\eqref{eq:corrector-01} attains the same objective value as $f^\ast$. Therefore, $p^{\ast} = f^{\ast}$.
 	Note that 
 	\begin{equation}
 		\tilde{\MZ} = \MZ^\ast + \Delta,
 	\end{equation}
 	where $\Delta = [\delta_1, \delta_2, \ldots \delta_n]$, which equals $\MR^\ast \MX + \vt^{\ast}\ones_n\tran - \MZ^{\ast}$.

\end{proof}

\section{Architecture Details}
\label{sec:app:architecture}

\begin{figure*}[bt!]
   \centering
   \includegraphics[width=0.7\linewidth]{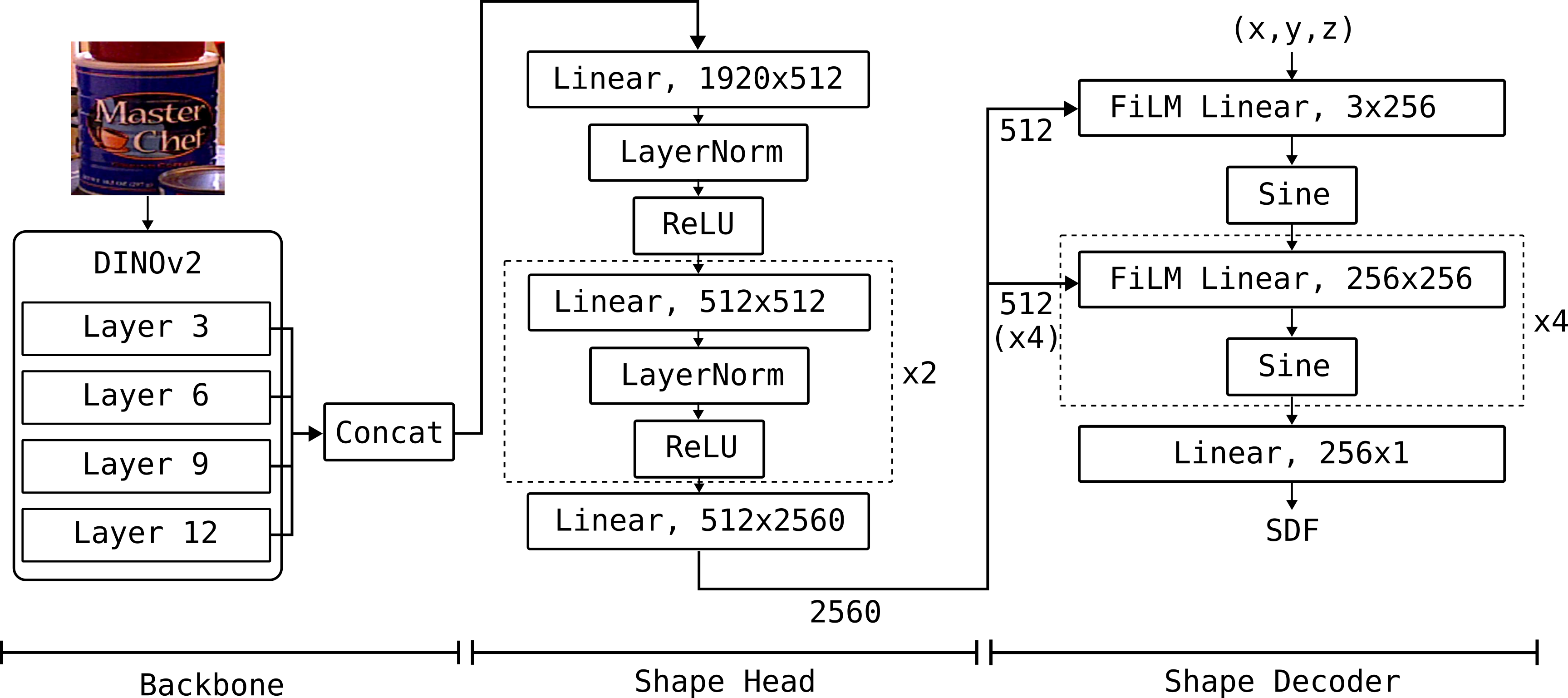}
   \caption{Diagram of our architecture for \pipelineName's shape head and shape decoder.}
   \label{fig:app:shape-branch}
   \vspace{-2mm}
\end{figure*}

\begin{figure*}[bt!]
   \centering
   \includegraphics[width=0.6\linewidth]{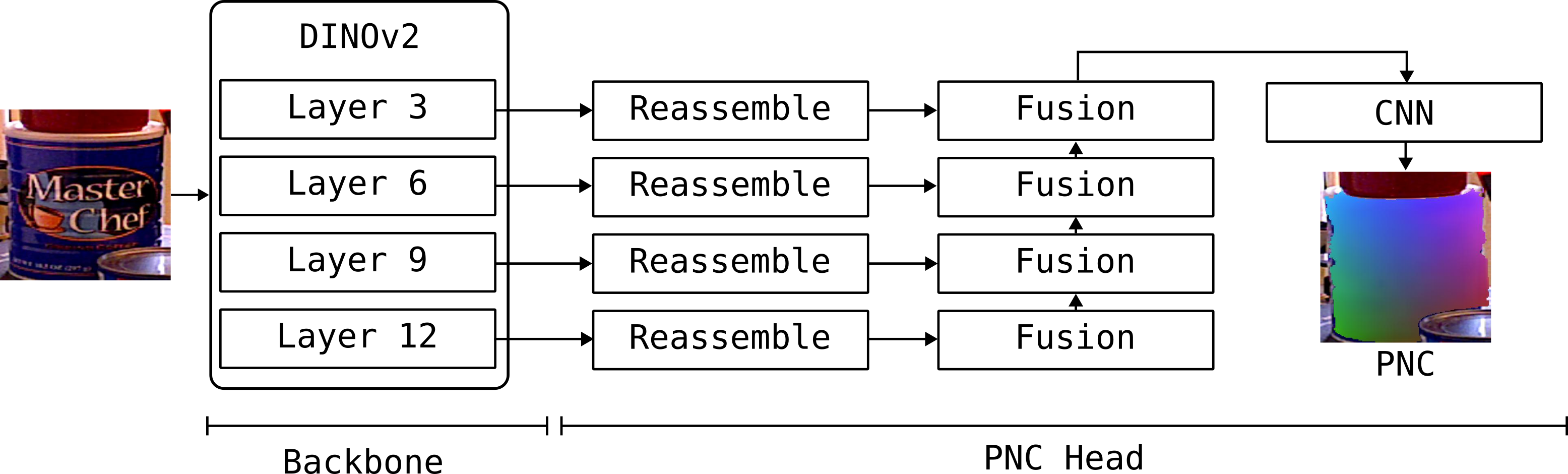}
   \caption{Diagram of our architecture for \pipelineName's PNC head.}
   \label{fig:app:pnc-branch}
   \vspace{-2mm}
\end{figure*}

Fig.~\ref{fig:app:shape-branch} and Fig.~\ref{fig:app:pnc-branch} show the detailed architecture of our \pipelineName.
We use the ViT-S variant of DINOv2~\cite{Oquab23arxiv-dinov2} as our ViT backbone, and keep it frozen during training.

The shape head (Fig.~\ref{fig:app:shape-branch}) is an MLP with layer normalization and ReLU activations. The network has two hidden layers, with hidden dimension equals to $512$.
We take the \texttt{[cls]} tokens from layers $(3, 6, 9, 12)$ as well as the patch tokens from layer $12$, concatenate them into a $1920$-dimensional vector and feed them into the shape head.
The output of the shape head is a $2560$-dimensional vector, which we use to condition the shape encoder.
The shape decoder is a FiLM-conditioned MLP with sinusoidal activations~\cite{Perez18aaai-FiLMVisual,Sitzmann20neurips-siren}.
The network has $4$ hidden layers, with hidden dimension equals to $256$.
The output of the shape head is split into $5$ $512$-dimensional vectors, 
with each vector conditioning one layer of the shape decoder.
The final layer is a linear layer with output dimension equals to $1$.

The PNC head (Fig.~\ref{fig:app:pnc-branch}) is based on DPT~\cite{Ranftl21iccv-DPT}. 
We take tokens from layers $(3, 6, 9, 12)$ and pass them independently through reassemble blocks (see~\cite{Ranftl21iccv-DPT} for details).
We then pass the output of each reassemble blocks through fusion blocks respectively.
We use projection as the readout operation and generate features with $256$ dimensions.
Note that different from the original DPT Implementation~\cite{Ranftl21iccv-DPT}, 
we use group normalization in the residual convolutional unit in each of the fusion blocks.
We note that this helps with the stability of the network in training.
Lastly, we use a CNN to produce the PNC output.
The CNN has $3$ layers, with the first layer having $256$ input channels and $128$ output channels, the second layer having $128$ input channels and $32$ output channels, and the third layer having $32$ input channels and $3$ output channels.
The kernel sizes of the first two layers are $3\times 3$, and the kernel size of the third layer is $1\times 1$.
We use ReLU as the activation function.

\section{Training Details}
\label{sec:app:training-details}

\subsection{Supervised Training Details}
For PNC, we adopt a loss similar to the soft-$L_1$ used in~\cite{Wang19-normalizedCoordinate}:
\begin{align}
&L_{PNC}\left(\MZ, \MZ^*\right) \nonumber \\
&=\frac{1}{n}\sum_{i=1}^N\left\{\begin{array}{ll}\left(\vz_i-\vz_i^*\right)^2 / \left( 2 \zeta \right), & \left|\vz_i-\vz_i^*\right| \leq \zeta \\ \left|\vz_i-\vz_i^*\right|-\zeta / 2, & \left|\vz_i-\vz_i^*\right|>\zeta\end{array}\right.
\label{eq:app:pnc-sup-loss}
\end{align}
where $\zeta$ is the threshold to control switching between $L_1$ and quadratic loss. We use $\zeta = 0.1$ for all datasets.

For shape head and decoder, we adopt a loss similar to~\cite{Sitzmann20neurips-siren}:
\begin{align} 
&\mathcal{L}_{SDF} \nonumber \\
&=\frac{\gamma_1}{M_1}\sum_{i \in \Omega_0, |\Omega_0| = M_1} | f_d(\vxx_i~|~\vh) - f_d(\vxx_i~|~\vh)^* | \\
&+ \frac{\gamma_2}{M_2} \sum_{i \in \Omega_e, |\Omega_e| = M_2} \phi\left( f_d(\vxx~|~\vh)\right) d \vxx \nonumber \\
&+ \frac{\gamma_3}{M_1 + M_2}\sum_{i \in \Omega_0 \cup \Omega_e}\left|\left\|\nabla_{\vxx} f_d(\vxx_i~|~\vh)\right\|-1\right|  
\label{eq:app:sdf-sup-loss}
\end{align}
where $\psi(\mathbf{x})=\exp (-\alpha \cdot|\Phi(\mathbf{x})|)$ and $\alpha =100$,
$\Omega_0$ is a set of points sampled on the shape manifold, equivalent to the zero-level set of SDF (with size $M_1$) and 
$\Omega_e$ is the set of points sampled outside the shape manifold, but inside the domain in which we want to reconstruct the SDF values (with size $M_2$).
The first term utilizes supervision directly on the SDF values.
The second term penalizes off-manifold points to have close-to-zero SDF values.
The last term in~\eqref{eq:app:sdf-sup-loss} is the Eikonal regularization term where we ensure the norm of spatial gradients to be $1$ almost everywhere in $\Omega$.
$\gamma_1$, $\gamma_2$, and $\gamma_3$ are hyperparameters.
We compute the ground truth SDF values using~\texttt{libigl}\footnote{https://github.com/libigl/libigl}.

The overall loss used is the following:
\begin{align}
    \mathcal{L}_{total} = \alpha \mathcal{L}_{PNC} + \beta \mathcal{L}_{SDF} 
\end{align}
where $\alpha$ and $\beta$ are hyperparameters.

\paragraph{YCBV Dataset} 
For the YCBV dataset,
we train two supervised models: \pipelineNameSyn and \pipelineNameReal.
For \pipelineNameSyn, 
we train using synthetic rendered images of the YCBV objects using BlenderProc~\cite{Denninger20rss-blenderproc}.
We generate a total of $4200$ images, with $200$ images per object.
We use $\alpha = 5\times 10^{3}$, $\beta = 0.1$, 
$\gamma_1 = 3\times 10^{3}$, $\gamma_2 = 2\times 10^{2}$, and $\gamma_3 = 50$.
We train \pipelineNameSyn with an Adam optimizer and learning rate of $3 \times 10^{-4}$ for 500 epochs.
We use a batch size of $16$ and weight decay of $10^{-5}$.

For \pipelineNameReal,
we train with the provided real-world train set images.
We use $\alpha = 5\times 10^{3}$, $\beta = 0.1$, 
$\gamma_1 = 3\times 10^{3}$, $\gamma_2 = 2\times 10^{2}$, and $\gamma_3 = 50$.
We train \pipelineNameReal with an Adam optimizer with a learning rate of $3 \times 10^{-4}$ for 50 epochs,
We use a batch size of $10$ and weight decay of $10^{-5}$.

\paragraph{SPE3R Dataset}
We train \pipelineName on the train set of SPE3R. 
We use an Adam optimizer with a learning rate of $10^{-4}$, weight decay of $10^{-6}$, a batch size of 16 and for 100 epochs.
We use a cosine annealing scheduler with restarts per $4000$ iterations and a multiplication factor of $2$.
We use $\alpha = 50$, $\beta = 1$, 
$\gamma_1 = 3\times 10^{3}$, $\gamma_2 = 2\times 10^{2}$, and $\gamma_3 = 50$.

\paragraph{NOCS Dataset}
We train \pipelineName on the train sets of \NOCSCAMERA and \NOCSREAL. We use an Adam optimizer with a learning rate of $10^{-4}$, weight decay of $10^{-6}$, a batch size of 16 and for 50 epochs.
We use a cosine annealing scheduler with restarts per $4000$ iterations and no multiplication factor.
We use $\alpha = 50$, $\beta = 1$, 
$\gamma_1 = 3\times 10^{3}$, $\gamma_2 = 2\times 10^{2}$, and $\gamma_3 = 50$.

\subsection{Self-Training Details}
\paragraph{Correction} For self-training, we use Alg.~\ref{alg:corrector} and Alg.~\ref{alg:corrector-lsq}.
There are two main components to both Alg.~\ref{alg:corrector} and Alg.~\ref{alg:corrector-lsq}: 
the solver for $\MZ$ (Line~\ref{algo:nocs-update} in Alg.~\ref{alg:corrector} and Alg.~\ref{alg:corrector-lsq})
and the solver for $\vh$ (Line~\ref{algo:shape-update} in Alg.~\ref{alg:corrector} and Alg.~\ref{alg:corrector-lsq}).

For the gradient descent solver for $\MZ$ in both Alg.~\ref{alg:corrector} and Alg.~\ref{alg:corrector-lsq},
we use a step size of $1 \times 10^{-3}$
and a maximum number of iterations of $50$
for all datasets.

For the solver for $\vh$ in Alg.~\ref{alg:corrector},
we use projected gradient descent as described in Section~\ref{sec:arch-corrector} with a step size of $1 \times 10^{-2}$ with a maximum number of iterations of $25$ for all datasets.
In addition, as described in Remark~\ref{rem:shape-multi-view},
we store $\MZ$ in a buffer from multiple views.
The buffer has a maximum size of $50$ frames.

\paragraph{Certification}
We use~\eqref{eq:oc-cert} as a boolean indicator function to generate pseudo-labels. 
For the YCBV dataset, we use $\epsilon = 1 \times 10^{-2}$ and $p = 0.98$.
For the SPE3R dataset, we use $\epsilon = 2 \times 10^{-2}$ and $p = 0.97$.

\paragraph{Self-Training}
We use~\eqref{eq:self-training-loss} for the loss function for self-training.
We use two separate stochastic gradient descent optimizers for the shape head and PNC head. 
For the shape head, we use a learning rate of $4 \times 10^{-4}$ for YCBV and $3 \times 10^{-4}$ for SPE3R. 
For the PNC head, we use a learning rate of $3 \times 10^{-4}$ for both YCBV and SPE3R.
We use a weight decay of $10^{-5}$ for both shape head and PNC head for all datasets.
We use a batch size of $3$ for the YCBV dataset and $10$ for the SPE3R dataset.

\section{Additional Results}
\label{sec:app:add-exp}

\begin{figure}[t!]
\centering
\begin{subfigure}{.4\linewidth}
\includegraphics[width=\linewidth]{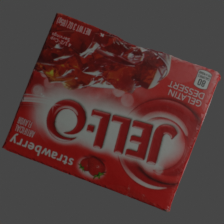}
\end{subfigure}%
~
\begin{subfigure}{.4\linewidth}
\includegraphics[width=\linewidth]{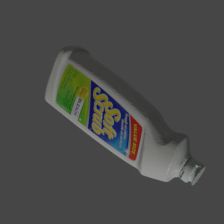}
\end{subfigure}
\caption{Two sample images from our YCBV synthetic dataset.}
\label{fig:app:ycbv-synthetic-render}
\vspace{-4mm}
\end{figure}

Similar to Section~\ref{sec:expt}, for tables, we use bold fonts for the best results and color the top three results with \tabfirstbox, \tabsecondbox~and \tabthirdbox, respectively. 
For tables fewer than three entries, we label only the top entry with \tabfirstbox~and bold font.

\subsection{YCBV Dataset}
\begin{figure}[t!]
\centering
  \centering
  \includegraphics[width=0.9\linewidth]{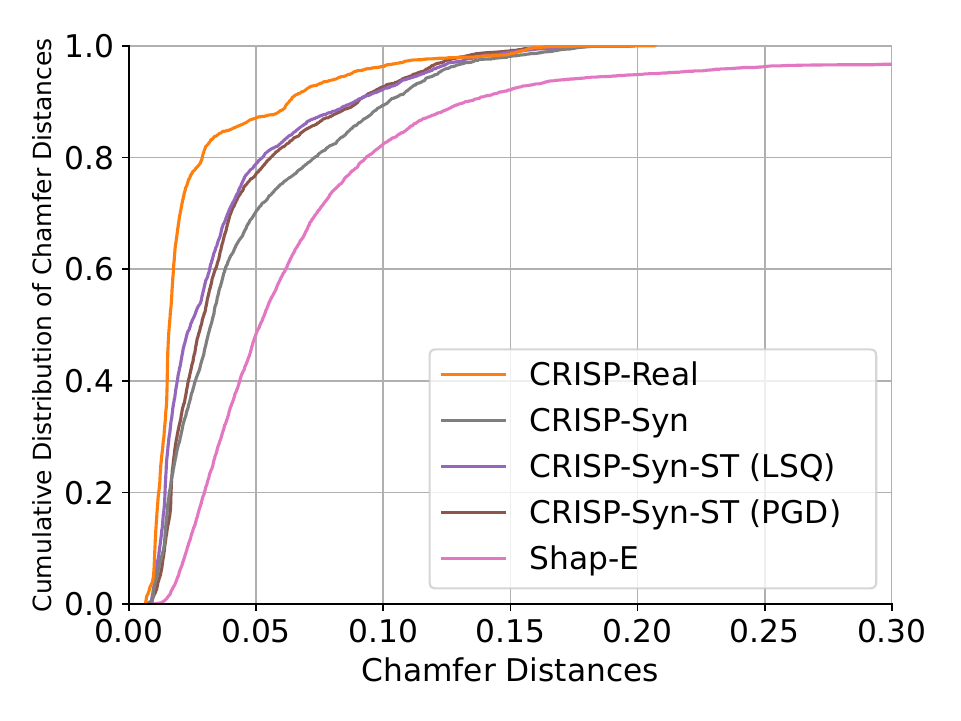}
  \caption{Cumulative distribution function of the Chamfer distances of the reconstructed shapes on the YCBV dataset.}
  \label{fig:ycbv-sota-shape-chamfer}
  \vspace{-3mm}
\end{figure}
\begin{figure}[t!]
  \centering
  \includegraphics[width=0.9\linewidth]{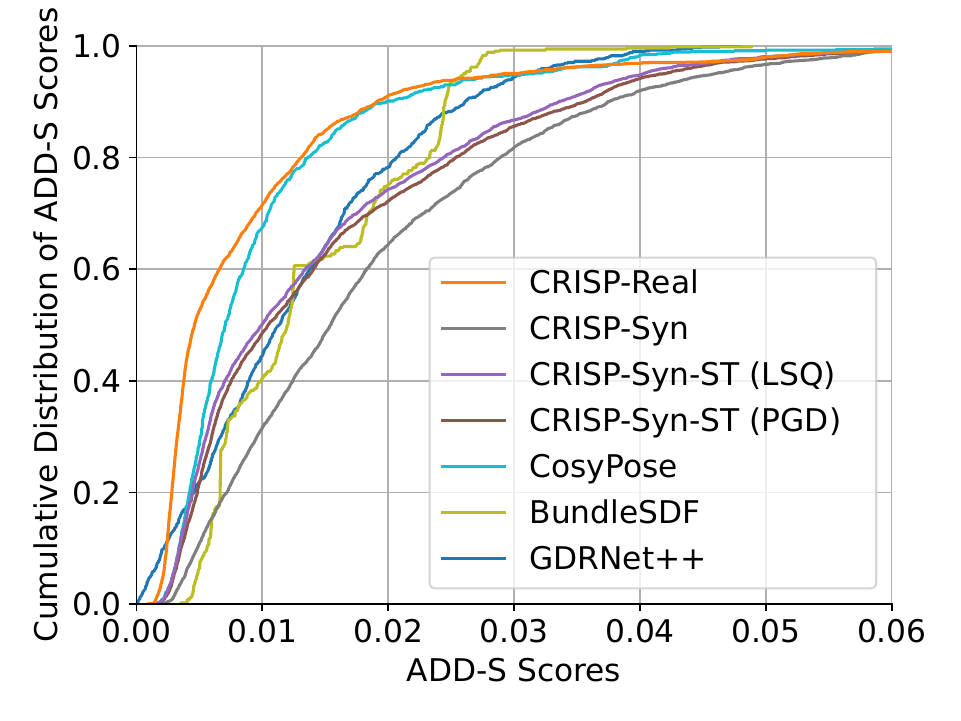}
  \caption{Cumulative distribution function of the ADD-S scores on the YCBV dataset.}
  \label{fig:ycbv-sota-adds}
  \vspace{-3mm}
\end{figure}

Fig.~\ref{fig:app:ycbv-synthetic-render} shows two sample images from our synthetic dataset.
Fig.~\ref{fig:ycbv-sota-shape-chamfer} shows the cumulative distribution function (CDF) curves of Chamfer distance between the reconstructed shape and the ground truth CAD model.
We note that \pipelineNameReal achieves the best performance,
with self-training improving the shape reconstruction performance of \pipelineNameSyn.
Fig.~\ref{fig:ycbv-sota-adds} shows the CDF curves of ADD-S scores.
We note that \pipelineNameReal again achieves the best performance,
with GDRNet++ achieves a slight performance edge at low ADD-S scores.
Self-training improves the performance of \pipelineNameSyn, with \pipelineNameSynLSQAdapt slightly outperforming \pipelineNameSynPGDAdapt.

\subsection{SPE3R Dataset}

\begin{table}[t!]
	\small
\centering
 \begin{tabular}{lllll} 
	\toprule
 Method & \multicolumn{2}{c}{$e^{L_1}_{shape} \downarrow$} & \multicolumn{2}{c}{$e^{L_1}_{pose} \downarrow$} \\ \cmidrule(lr){2-3} \cmidrule(lr){4-5}
 			 & Mean & Median & Mean & Median \\ \midrule
 \pipelineName  					& 0.163	    &  0.159	   & 0.292 	& 0.211	\\
  + Corrector 	& \tabfirst 0.139 	& \tabfirst 0.120 	& \tabfirst 0.191 	& \tabfirst 0.176 \\
 \bottomrule
 \end{tabular}
 \caption{Evaluation of \pipelineName with and without the corrector (Alg.~\ref{alg:corrector}) on the SPE3R dataset.}
 \label{tab:app:spe3r-results}
 \vspace{-3mm}
\end{table}

Our corrector (Alg.~\ref{alg:corrector}) can be used to improve performance during inference as well, not only for self-training.
Tab.~\ref{tab:app:spe3r-results} shows the performance of \pipelineName with and without the corrector (Alg.~\ref{alg:corrector}).
With the corrector, we see improved shape reconstruction performance and pose estimation performance.

\subsection{NOCS Dataset}

\begin{table*}[t!]
\centering
\small
\begin{tabular}{llllllll}
\toprule
        & \multicolumn{7}{c}{$e_{shape}$ (mm) $\downarrow$}    \\ \cmidrule(l){2-8} 
Methods & Bottle & Bowl & Camera & Can  & Laptop & Mug  & Avg  \\ \midrule
SPD~\cite{Tian20eccv-SPD}     & 3.44   & 1.21 & 8.89   & 1.56 & 2.91   & 1.02 & 3.17 \\
SGPA~\cite{Chen21iccv-SGPA}    & 2.93   & 0.89 & \tabthird 5.51   & 1.75 & \tabthird 1.62   & 1.12 & 2.44 \\
CASS~\cite{Chen20iccv-learningCanonicalShape}    & \tabsecond 0.75   & \tabsecond 0.38 & \tabsecond 0.77   & \tabsecond 0.42 & 3.73   & \tabsecond 0.32 & \tabsecond 1.06 \\
RePoNet~\cite{Ze22neurips-wild6d} & \tabthird 1.51   & \tabthird 0.76 & 8.79   & \tabthird 1.24 & \tabsecond 1.01   & \tabthird 0.94 & \tabthird 2.37 \\
\pipelineName    & \tabfirst 0.55   & \tabfirst 0.31 & \tabfirst 0.37   & \tabfirst 0.18 & \tabfirst 0.51   & \tabfirst 0.16 & \tabfirst 0.35 \\ \bottomrule
\end{tabular}
\caption{Shape reconstruction results on the \NOCS dataset.}
\label{tab:app:nocs-shape}
\end{table*}

\begin{table*}[t]
\centering
\small
\begin{tabular}{lllllllll}
\toprule
&             & \multicolumn{7}{c}{mAP}                                                          \\ \cmidrule(l){3-9} 
&Methods      & $IoU_{25}$ & $IoU_{50}$ & $IoU_{75}$ & \makecell{$5^{\circ}$ \\ $5 \mathrm{~cm}$} & \makecell{$5^{\circ}$ \\ $10 \mathrm{~cm}$} & \makecell{$10^{\circ}$ \\ $5 \mathrm{~cm}$} & \makecell{$10^{\circ}$ \\ $10 \mathrm{~cm}$} \\ \midrule
\parbox[t]{2mm}{\multirow{7}{*}{\rotatebox[origin=c]{90}{\small Category-level}}} &NOCS~\cite{Wang19-normalizedCoordinate}  & \tabfirst 84.8   & 78.0     & 30.1   & 10.0         & 9.8         & 25.2        & 25.8         \\
&Metric Scale~\cite{Lee21ral-metricScale} & 81.6   & 68.1   & \textendash       & 5.3        & 5.5         & 24.7        & 26.5         \\
&SPD~\cite{Tian20eccv-SPD}          & 81.2   & 77.3   & 53.2   & 21.4       & 21.4        & 54.1        & 54.1         \\
&CASS~\cite{Chen20iccv-learningCanonicalShape}         & \tabsecond 84.2   & 77.7   & \textendash      & 23.5       & \tabsecond 23.8        & 58.0      & \tabthird 58.3         \\
&SGPA~\cite{Chen21iccv-SGPA}         & \textendash        & \tabthird 80.1   & \tabsecond 61.9   & \tabsecond 39.6       & \textendash     & 70.7        & \textendash     \\
&RePoNet~\cite{Ze22neurips-wild6d}      & \textendash        & \tabsecond 81.1   & \textendash      & \tabfirst 40.4       & \textendash     & \tabfirst 68.8        & \textendash     \\
&SSC-6D~\cite{Peng22aaai-selfSupPoseShape}       & 83.2   & 73.0     & \textendash      & 19.6       & \textendash     & 54.5        & 56.2         \\\midrule
\parbox[t]{5mm}{\multirow{3}{*}{\rotatebox[origin=c]{90}{\makecell{\footnotesize Category\\-agnostic}}}} &DualPoseNet~\cite{Lin21iccv-dualposenet}  & \textendash        & 76.1   & 55.2   & \tabthird 31.3       & \textendash     & 60.4        & \textendash     \\
&FSD~\cite{Lunayach24icra-FSD}          & 80.9   & 77.4   & \tabsecond 61.9   & 28.1       & \tabfirst 34.4        & \tabthird 61.5        & \tabfirst 72.6         \\
&\pipelineName    & \tabsecond 84.2   & \tabfirst 83.5   & \tabfirst 70.5   & 22.5       & \tabthird 23.7        & \tabsecond 62.8        & \tabsecond 64.8        \\ \bottomrule
\end{tabular}
\caption{Pose estimation results on the \NOCS dataset.}
\label{tab:app:nocs-pose}
\end{table*}

Tab.~\ref{tab:app:nocs-shape} shows the shape reconstruction results on the NOCS dataset for all categories.
\pipelineName outperforms the baselines in all categories as well as the average in terms of $e_{shape}$.
Tab.~\ref{tab:app:nocs-pose} shows the pose estimation results on the NOCS dataset.
As noted in Section~\ref{sec:expt-nocs}, \pipelineName achieves comparable performance with the state-of-the-art category-agnostic methods.
Note that \pipelineName's rotation error performance can be potentially improved by incorporating more advanced loss functions such as the symmetry-aware loss used in~\cite{Wang19-normalizedCoordinate}.

\subsection{Additional Ablation Experiments}
\label{sec:app-ablation}

\paragraph{Scale Degeneracy.} 
\begin{table}[t!]
\small
\centering
\begin{tabular}{cccc}
\toprule
Ablations & \makecell{ADD-S \\(AUC)} $\uparrow$ & \makecell{ADD-S \\ (AUC)} $\uparrow$ & \makecell{ADD-S \\ (AUC)} $\uparrow$\\
& \multicolumn{1}{c}{1 cm} & \multicolumn{1}{c}{2 cm} & \multicolumn{1}{c}{3 cm}\\ \midrule
Normalized $\MZ$ (NOCS)                & $8 \times 10^{-4}$  & 0.036 & 0.095 \\
Unnormalized $\MZ$ (PNC)      & \tabfirst 0.22  & \tabfirst 0.42  & \tabfirst 0.54 \\\bottomrule
\end{tabular}
\caption{Comparing pose estimation performance between normalized $\MZ$ (NOCS) with PNC for self-training on the YCBV dataset.}
\label{app:tab:nocs-vs-pnc}
\end{table}

Comparing with NOCS~\cite{Wang19-normalizedCoordinate}, we see that PNC achieves significantly better pose estimation performance through the prevention of scale degeneracy (Tab.~\ref{app:tab:nocs-vs-pnc}).
To understand why the scale degeneracy is a problem, 
consider the pose optimization problem used with NOCS:
\begin{equation}
    (\hat{s}, \hat{\MR}, \hat{\vt}) = \underset{(s, \MR, \vt)}{\text{argmin}} \sum_{i=1}^{n}\norm{s\MR\vxx_i + \vt - \vz_i}^2.
    \label{eq:app:nocs-pose-opt}
\end{equation}
Assume noiseless depth $\{\vxx_i\}$, 
for a given $\vz_i$, we have $(\hat{s}, \hat{\MR}, \hat{\vt})$ as the solution to~\eqref{eq:app:nocs-pose-opt}.
If we scale $\vz_i$ by a factor of $b$, 
we can achieve the same loss for~\eqref{eq:app:nocs-pose-opt}
if we let $s = b \hat{s}$, $\MR = \hat{\MR}$, and $\vt = b\hat{\vt}$.
In practice, we observe such phenomenon (see Fig.~\ref{fig:app:nocs-scale-degeneracy}):
after self-training, the predicted $z_i$ becomes significantly smaller than the ground truth $z_i$ (Fig.~\ref{fig:app:nocs-scale-degen-small-nocs}).
If we transform the CAD model using the estimated pose, we observe that the transformed CAD model is much larger than the ground truth CAD model (Fig.~\ref{fig:app:nocs-scale-degen-large-cad}).
We note that this might be one of the reasons why prior works~\cite{Peng22aaai-selfSupPoseShape,Lunayach24icra-FSD} need to keep access to synthetic data during self-supervised training.

\begin{figure}[t!]
  \begin{subfigure}[b]{0.5\linewidth}
    \centering
    \includegraphics[width=0.6\linewidth]{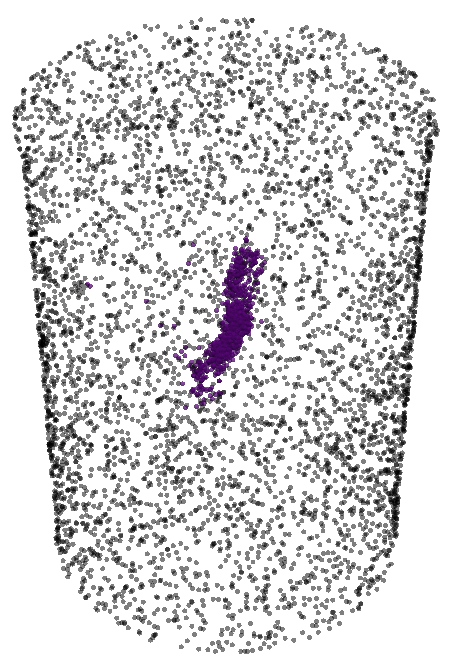}
    \caption{}
    \label{fig:app:nocs-scale-degen-small-nocs}
  \end{subfigure}%
  \begin{subfigure}[b]{0.5\linewidth}
    \centering
    \includegraphics[width=0.9\linewidth]{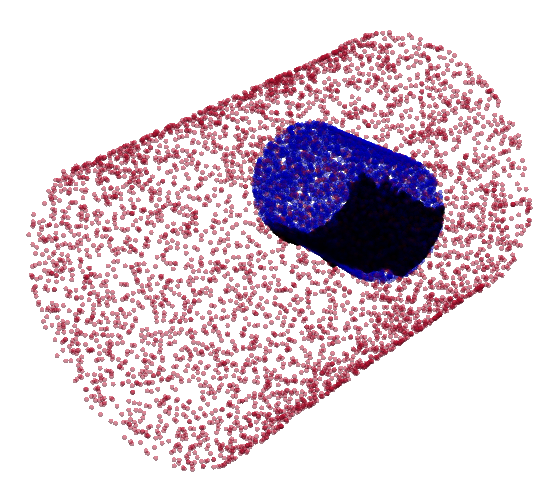}
    \caption{}
    \label{fig:app:nocs-scale-degen-large-cad}
  \end{subfigure}
  \caption{(a) \emph{Purple}: $\MZ$ predicted by the PNC head after self-training if we adopt normalization for $\MZ$ (equivalent to NOCS~\cite{Wang19-normalizedCoordinate}), 
  \emph{Black}: ground truth CAD model. 
  (b) \emph{Blue}: ground truth transformed CAD model. \emph{Black}: depth point cloud. 
  \emph{Red}: CAD model transformed using estimated pose with normalized $\MZ$.
  }
  \label{fig:app:nocs-scale-degeneracy}
  \vspace{-4mm}
\end{figure}

\paragraph{Use of $\vh = f_e(\calI)$ in~\eqref{eq:active-shape}.}
\begin{table}[t!]
\small
\centering
\begin{tabular}{cccc}
\toprule
Ablations & \makecell{$e_{shape}$ \\(AUC)} $\uparrow$ & \makecell{$e_{shape}$ \\ (AUC)} $\uparrow$ & \makecell{$e_{shape}$ \\ (AUC)} $\uparrow$\\
& \multicolumn{1}{c}{3 cm} & \multicolumn{1}{c}{5 cm} & \multicolumn{1}{c}{10 cm}\\ \midrule
Without $\vh = f_e(\calI)$       & 0.21  & 0.35 & 0.55 \\
With $\vh = f_e(\calI)$     & \tabfirst 0.25  & \tabfirst 0.43  & \tabfirst 0.65 \\\bottomrule
\end{tabular}
\caption{Comparing shape reconstruction performance between using and not using $\vh = f_e(\calI)$ in~\eqref{eq:active-shape} on the YCBV dataset.}
\label{app:tab:encoder-basis}
\end{table}

The use of $\vh = f_e(\calI)$ in~\eqref{eq:active-shape} can be seen as a way
to incorporate the shape encoder's output into Alg.~\ref{alg:corrector-lsq}.
We empirically observe that doing so improves the shape reconstruction performance after self-training (see Tab.~\ref{app:tab:encoder-basis}).

\paragraph{Normalization with $\MD$.}
As noted in Section~\ref{sec:llsq-active-shape-decoder}, 
we apply normalization to $\MF$ through the use of the diagonal matrix $\MD$.
To be more precise, we calculate $d_k$ (the diagonal entries of $\MD$) as the inverse of the diameter of the bounding box for each $f_d(\cdot~|~\vh_k)$ in~\eqref{eq:active-shape}.
Empirically, this improves the shape reconstruction performance after self-training (see Tab.~\ref{app:tab:lsq-normalization}).

\begin{table}[t!]
\small
\centering
\begin{tabular}{cccc}
\toprule
Ablations & \makecell{$e_{shape}$ \\(AUC)} $\uparrow$ & \makecell{$e_{shape}$ \\ (AUC)} $\uparrow$ & \makecell{$e_{shape}$ \\ (AUC)} $\uparrow$\\
& \multicolumn{1}{c}{3 cm} & \multicolumn{1}{c}{5 cm} & \multicolumn{1}{c}{10 cm}\\ \midrule
Without normalization       & 0.23  & 0.40 & 0.63 \\
With normalization     & \tabfirst 0.25  & \tabfirst 0.43  & \tabfirst 0.65 \\\bottomrule
\end{tabular}
\caption{Comparing shape reconstruction performance between using and not using $\MD$ for normalization on the YCBV dataset.}
\label{app:tab:lsq-normalization}
\end{table}

\end{document}